\documentclass{article} 
\usepackage{iclr2027_conference,times}

\usepackage{amsmath,amsfonts,bm}

\def\eqref#1{equation~\ref{#1}}

\def\1{\bm{1}}

\DeclareMathAlphabet{\mathsfit}{\encodingdefault}{\sfdefault}{m}{sl}
\SetMathAlphabet{\mathsfit}{bold}{\encodingdefault}{\sfdefault}{bx}{n}

\usepackage[most]{tcolorbox}
\usepackage{url}
\usepackage{amsthm}
\usepackage{algorithm}
\usepackage{algpseudocode}
\usepackage{booktabs}
\usepackage{array}
\usepackage{colortbl}
\usepackage{multirow}
\usepackage{graphicx}
\usepackage{float}
\usepackage{placeins}
\usepackage{titletoc}
\usepackage{hyperref}
\definecolor{avgcolumn}{RGB}{238,239,255}
\definecolor{evtarow}{RGB}{255,226,226}
\newcolumntype{N}{>{\centering\arraybackslash}m{1.25cm}}
\newcolumntype{A}{>{\columncolor{avgcolumn}\centering\arraybackslash}m{1.25cm}}
\newcolumntype{E}{>{\columncolor{evtarow}\centering\arraybackslash}m{1.25cm}}
\algrenewcommand\algorithmicrequire{\textbf{Input:}}
\algrenewcommand\algorithmicensure{\textbf{Output:}}

\newtheorem{theorem}{Theorem}

\title{Demonstration-Free Success-Probability Reward Learning for Generalist Robot Policies}

\author{
\textbf{Duo Wu}$^{\spadesuit}$, 
\textbf{Haifeng Wang}$^{\spadesuit,\heartsuit}$,
\textbf{Rongwei Lu}$^{\spadesuit}$,
\textbf{Jinghe Wang}$^{\spadesuit}$, 
\textbf{Tianyi Xiong}$^{\spadesuit}$,\\
\textbf{Zhimin Wang}$^{\spadesuit}$,
\textbf{Chao Yu}$^{\spadesuit}$, 
\textbf{Shuai Ma}$^{\diamondsuit}$,
\textbf{Zhi Wang}$^{\spadesuit}$ \\
\normalfont
$^{\spadesuit}$ Tsinghua University \quad
$^{\diamondsuit}$ Pengcheng Laboratory \quad
$^{\heartsuit}$ Yuanxing Robotics
}

\iclrfinalcopy
\begin{document}

\maketitle
\pagestyle{plain}
\vspace{-0.4cm}
\begin{abstract}

    Reinforcement learning (RL) enables generalist robot policies to improve through trial-and-error interaction, yet its effectiveness is fundamentally  constrained by sparse task rewards. Existing general-purpose reward models typically alleviate this issue by learning task progress from expert demonstrations, but introduce a distribution mismatch with the mixed-quality rollouts encountered during policy optimization, making their estimates unreliable on suboptimal and failed behaviors from which the policy must learn. In this work, we introduce a demonstration-free reward learning paradigm where dense reward feedback can be learned directly from sparse task outcomes and policy experience. We theoretically show  that terminal task outcomes implicitly define dense success-probability feedback at intermediate timesteps, which can be recursively learned through bootstrapping. Based on this insight, we introduce \textbf{\textit{eVTA$_0$}}, which learns success probabilities from mixed-quality policy rollouts through temporal-difference-style bootstrapping, without expert demonstrations or intermediate  annotations. We further introduce \textbf{\textit{RL with Evolving Rewards  (RLER)}}, a closed-loop framework that  adapts eVTA$_0$ using newly collected rollouts as the policy evolves. Experiments show that eVTA$_0$ provides more informative rewards than state-of-the-art reward models and achieves the best average policy performance across all LIBERO task suites under the same RL training budget, improving success rates by $5.4\%$--$13.8\%$ over the initial policy. In real-world manipulation, RLER further improves overall success rates by $20\%$--$26\%$, with $35\%$--$36\%$ gains under out-of-distribution conditions. These results demonstrate the effectiveness of demonstration-free reward learning and adapting rewards as the policy evolves. \textit{Project webpage: https://duowuyms.github.io/evta0}.

\end{abstract}

\section{Introduction}

Reinforcement learning (RL)~\citep{li2025gr,lu2025vla,li2025simplevla,luo2025precise} has recently shown great promise for enabling generalist robot policies, including vision-language-action models (VLAs)~\citep{intelligence2025pi06,kim2024openvla,li2024cogact,walke2023bridgedata} and world-action models (WAMs)~\citep{cen2025worldvla,chen2026lawam,wang2026openwam}, to acquire new capabilities beyond demonstrations through trial-and-error experience. However, effective RL training remains constrained by sparse task rewards. In robotic manipulation, rewards are often available only as terminal success or failure signals~\citep{mudrs,rengarajan2022reinforcement,luo2025precise,chen2025conrft,zhouefficient}, providing little guidance on which intermediate behaviors contribute to task completion. This problem becomes particularly severe in long-horizon tasks that require hundreds or even thousands of actions before completion~\citep{xu2025stare,sermanet2024robovqa}, making credit assignment and policy improvement difficult.


Recent studies have explored general-purpose reward models that estimate task progress from intermediate observations across diverse tasks and scenarios~\citep{zhai2025vision,shu2025rftf,liu2026fly,bai2025evolve,lee2026roboreward}, offering a promising approach to alleviating reward sparsity. Specifically, they follow a common paradigm of training reward models on expert demonstrations to estimate how much an action advances the task toward completion, under the assumption that task completion progresses monotonically over time~\citep{din2025vision,lee2026roboreward}. Nevertheless, this  paradigm introduces a distribution mismatch with the mixed-quality rollouts encountered during RL training, which naturally contain suboptimal and failed trajectories. As a result, their reward estimates can become unreliable on the suboptimal and failed   behaviors from which the  policy needs to learn. These limitations motivate a demonstration-free paradigm for reward learning.

The central finding of this paper is that dense rewards can be learned directly from mixed-quality policy rollouts without expert demonstrations, opening a new paradigm for reward learning. We first establish the theoretical foundation behind this finding by showing that sparse terminal outcomes implicitly define dense success-probability feedback at intermediate timesteps. Specifically, the intermediate reward can be formulated as the probability of eventual task success conditioned on the trajectory history. Such success probabilities can be recursively learned through bootstrapping, allowing terminal supervision to be propagated backward without requiring  intermediate  annotations.  
Building on this principle, we design \textbf{\textit{eVTA}$_0$}\footnote{\textit{eVTA} stands for \textit{embodied ventral tegmental area (VTA)}. Inspired by the role of the biological VTA in reward processing and learning, we aim to build an embodied VTA that generates informative reward signals for continual improvement of generalist robot policies. eVTA$_0$ represents the first step toward this vision.}, a reward model  that observes manipulation histories and estimates intermediate success probabilities. We develop a simple temporal-difference (TD) style bootstrapping recipe to train eVTA$_0$ using only terminal signals from mixed-quality rollouts. Furthermore, enabled by the demonstration-free learning ability of eVTA$_0$, we introduce \textbf{\textit{RL with Evolving Rewards  (RLER)}}, where newly collected policy rollouts are  used to update eVTA$_0$, and eVTA$_0$ in turn provides adapted rewards for subsequent policy optimization. In this way, reward learning can evolve with the changing behavior distribution of the policy.

Through extensive experiments, we show that eVTA$_0$  produces reward signals that more reliably track task evolution and distinguish successful from failed trajectories than state-of-the-art reward models. 
When integrated into RL training on LIBERO~\citep{fei2025libero}, eVTA$_0$ achieves the best average success rate across all  task suites under the same training budget, improving over the initial policy by $5.4\%$--$13.8\%$.
Moreover, in real-world manipulation, RLER progressively improves policy performance, increasing overall success rates by $20\%$--$26\%$  over the initial SFT policies. The performance gains are more evident under out-of-distribution (OOD) conditions, with improvements  of $35\%$--$36\%$ and  $15\%$--$20\%$ over the SFT policy and fixed eVTA$_0$ baseline, respectively. 

The contributions of this paper are summarized as follows:

\begin{itemize}
    \item We theoretically reveal that terminal task outcomes implicitly define intermediate reward signals, where rewards can be modeled as conditional task success probabilities and recursively learned through bootstrapping.

    \item Based on this principle, we establish  a demonstration-free reward learning paradigm by developing eVTA$_0$, a general-purpose reward model trained through a simple TD-style bootstrapping recipe to estimate success probabilities from manipulation histories.
    
    \item We further introduce RL with Evolving Rewards (RLER), a closed-loop learning framework that adapts eVTA$_0$ using newly collected policy rollouts, allowing reward feedback to evolve with the changing behavior distribution of the policy.

    \item We conduct extensive evaluations on both reward modeling and policy learning. 
    eVTA$_0$ consistently provides high-quality reward feedback and improves policy performance under the same RL training budget. Real-world experiments further show that RLER improves overall success rates by $20\%$--$26\%$  over the initial SFT policies, while achieving $35\%$--$36\%$ improvements under OOD conditions. 
\end{itemize}

\section{Related Work}
Sparse terminal rewards hinder effective RL for training generalist robot policies, motivating a broad range of reward densification methods that can be grouped into four categories.

\noindent\textbf{Handcrafted dense rewards.} 
A common strategy is to manually design dense rewards based on heuristic signals~\citep{li2025reflection,ye2025vla,pan2026sa,team2026gigabrain,wang2025vla}, such as sub-goal completion~\citep{zhang2025reinbot}, end-effector states~\citep{xu2025stare}, or object-gripper distances~\citep{liu2026towards}. 
Although effective in specific settings, these methods require substantial engineering effort and often fail to generalize beyond the target tasks or scenarios~\citep{cobbe2020leveraging,tan2025robo}.  

\noindent\textbf{Stage-based rewards.}
To reduce reward engineering overhead, another line of work exploits the stage structure of manipulation tasks to construct dense rewards~\citep{lu2025vla,mudrs,escoriza2025multi,chen2025sarm}. 
For example, SARM~\citep{chen2025sarm} predicts task stages using labels derived from subtask annotations. 
However, these approaches still require manually specifying task stages, defining stage indicators, or providing subtask annotations, which may need to be redesigned for new tasks or scenarios and thus limit scalability to diverse  settings. 
Our work instead learns success-probability feedback without requiring task-specific decomposition.

\noindent\textbf{World-model-derived  rewards.} A separate  line of work leverages world models trained on expert demonstrations to derive dense rewards~\citep{xiao2025world,liu2026world,liu2026dual,zhu2025wmpo}. These methods typically learn to reconstruct observations or predict future states, and then derive rewards from reconstruction errors~\citep{li2025vla} or similarity to ground-truth observations~\citep{hung2025nora}. 
The major obstacle of these methods lies in their reliance on modeling environment dynamics and  high-quality demonstrations.
In comparison, our work directly models success probabilities without learning full environment dynamics or relying on expert demonstrations.

\noindent\textbf{General-purpose reward models.}
With the growing capabilities of foundation models, recent studies have started to explore their adaptation as general-purpose reward models that estimate task progress across diverse manipulation tasks~\citep{lee2026roboreward,tan2025robo,wang2026world,huang2026rynnvalue}. These methods can be categorized into training-free and training-based approaches. 
Training-free methods, such as GVL~\citep{ma2024vision} and TOPReward~\citep{chen2026topreward}, prompt frozen vision-language models (VLMs) to infer task progress from video frames, but they  tend to produce noisy or inconsistent reward signals~\citep{zhang2025rewind}. 
To improve reward quality and stability, training-based methods instead learn progress prediction from large-scale robotic data. 
They pretrain vision foundation models on expert demonstrations using temporal positions of observations as pseudo progress labels, assuming that task completion evolves monotonically over time. However, this assumption becomes restrictive during policy learning, where rollouts naturally contain both suboptimal and failed trajectories with non-monotonic progress.
Existing methods therefore introduce additional data construction and training strategies to improve robustness. 
For example, VLAC~\citep{zhai2025vision} employs pair-wise difference filtering, while Robometer~\citep{liang2026robometer} introduces inter-trajectory preference comparisons to incorporate suboptimal and failed trajectories. Nevertheless, these approaches involve complex pipelines and remain dependent on the quality and coverage of expert demonstrations.
In comparison, eVTA$_0$ introduces a demonstration-free reward learning paradigm that shifts reward modeling from demonstration-based progress estimation to experience-driven success-probability learning. It directly learns intermediate success probabilities from mixed-quality policy rollouts using only terminal outcomes, without requiring expert demonstrations or manually constructed intermediate progress annotations.

\section{Methodology}
\label{sec:method}
In this section, we introduce eVTA$_0$ and its integration into RL training. We first establish the theoretical foundation of success-probability feedback, then describe the design and training method of eVTA$_0$, and finally introduce RLER for adapting reward learning alongside policy optimization.

\subsection{Theoretical Foundation}
\label{sec:formulation}



\paragraph{Reward definition.}
Let $\mathcal O$  denote the observation 
space. For a rollout of length $T$, let
$\omega=(o_0,o_1,\ldots,o_T)\in\Omega$,
where $o_t\in\mathcal O$, and $\Omega$ denotes
the space of all such rollouts. Let
$H_t=(o_0,o_1,\ldots,o_t)$ denote the trajectory prefix
up to timestep $t$.
A given rollout policy $\pi$ and the environment induce a probability space
$(\Omega,\mathcal F,\mathbb P_\pi)$, where $\mathcal F$ is the event space
and $\mathbb P_\pi$ is the rollout distribution. Let
$\mathcal F_t:=\sigma(H_t)$ denote the information revealed by this prefix,
with $\mathcal F_t\subseteq\mathcal F_{t+1}$. Each rollout terminates at $T$
and provides only a binary terminal success label $S\in\{0,1\}$, so the
terminal sparse reward is $r_T:=S$. 
At each intermediate timestep, we define the reward as the  probability of eventual task success conditioned on the information revealed by the trajectory prefix:
\begin{equation}
    \label{eq:success-probability-reward}
    r_t:=\mathbb E_\pi[S\mid\mathcal F_t]
=\mathbb P_\pi(S=1\mid\mathcal F_t).
\end{equation}


\begin{theorem}[Temporal consistency of success-probability rewards]
\label{theorem:martingale}
The success probability reward defined above satisfies
$r_t=\mathbb E_\pi[r_T\mid\mathcal F_t]$ for every $t\le T$.
Moreover, for every $t<T$, it satisfies the temporal consistency property:
\begin{equation}
\label{eq:martingale}
\mathbb E_\pi[r_{t+1}\mid\mathcal F_t]=r_t.
\end{equation}
\end{theorem}
The temporal consistency property above suggests that sparse terminal rewards can be propagated backward through one-step bootstrapping. The following theorem further shows that such bootstrapping consistently recovers the underlying success-probability reward.


\begin{theorem}[Population consistency of one-step bootstrapping]
\label{theorem:bootstrap-consistency}
Consider the ideal backward regression defined by the terminal anchor
$p_T^\star:=S$ and, for each $t<T$:
\begin{equation}
\label{eq:one-step-objective}
p_t^\star
\in
\arg\min_{g\in L_2(\mathcal F_t)}
\mathbb E_\pi\left[\left(g-p_{t+1}^\star\right)^2\right],
\end{equation}
where $L_2(\mathcal F_t)$ denotes the square-integrable
$\mathcal F_t$-measurable candidate predictions. Then the population minimizer
is unique up to almost-sure equality and satisfies, for every $t\le T$:
\begin{equation}
\label{eq:one-step-optimal}
p_t^\star=r_t=\mathbb E_\pi[S\mid\mathcal F_t].
\end{equation}
Thus, terminal supervision can be propagated backward through ideal one-step
regression without changing the population success-probability target.
\end{theorem}
This result provides theoretical justification for deriving intermediate success-probability feedback from sparse terminal outcomes: although only terminal signals are available, the success probability at intermediate timesteps can be recursively characterized through one-step bootstrapping.
Detailed proofs of both theorems are provided in Appendix~\ref{appendix:math-proof}.

\subsection{Design}

Building upon the  formulation in Section~\ref{sec:formulation}, we introduce eVTA$_0$ to approximate the probability of eventual task success conditioned on the trajectory history. Specifically, eVTA$_0$  approximates the ideal success-probability reward $p_t^\star$ in Theorem~\ref{theorem:bootstrap-consistency} with a learnable function $\hat p_\theta$ parameterized by $\theta$. Unlike existing general-purpose reward models that estimate task progress, eVTA$_0$ directly models eventual task success, making its prediction naturally aligned with sparse terminal binary outcomes while providing dense intermediate feedback for policy learning. The details of eVTA$_0$ are described below.

\begin{figure}
    \centering
    \includegraphics[width=0.97\textwidth]{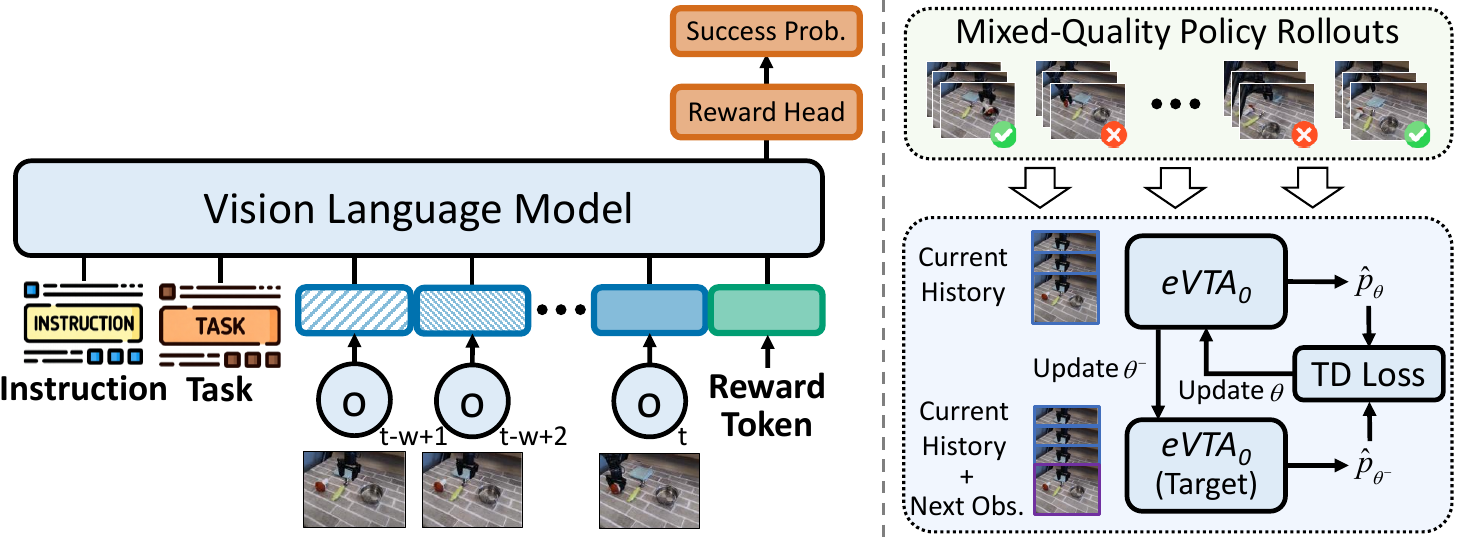}
    \vspace{-0.3cm}
    \caption{
Overview of eVTA$_0$. The left panel illustrates the model architecture, which estimates the probability of eventual task success conditioned on the task description and recent observation history. The right panel illustrates the training procedure, where eVTA$_0$ is trained on mixed-quality policy rollouts through TD-style bootstrapping.
}
    \label{fig:evta_framework}
\end{figure}

\noindent\textbf{Model architecture.}
To approximate the success probability  in Eq.~\ref{eq:success-probability-reward}, eVTA$_0$ predicts eventual task success from task descriptions and recent observation histories. Since maintaining the complete trajectory prefix is computationally expensive, eVTA$_0$ uses a sliding context window  $\mathcal{O}_t=\{o_{t-w+1},\ldots,o_t\}$ of $w$ recent observations as the input, where $w$ denotes the window size. The  description of the target manipulation task $\ell$ is additionally provided to enable task-conditioned success prediction  across diverse  tasks.

Figure~\ref{fig:evta_framework} depicts the architecture of eVTA$_0$. To infer success probabilities from visual trajectories and task objectives, eVTA$_0$ is built upon a VLM to exploit its visual and linguistic understanding  capabilities. The input sequence is formally defined as:
\begin{equation}
x_t=[I,\ell,\mathcal{O}_t,\texttt{[REWARD]}],
\end{equation}
where $I$ is the instruction prompt specifying the success-probability prediction task, $\ell$ describes the target manipulation task, $\mathcal{O}_t$ denotes the observation history, and \texttt{[REWARD]} is a special token that queries the model for the success probability.
Let $F_\psi$ denote the VLM backbone. We extract the hidden representation corresponding to the \texttt{[REWARD]} token, $h_t=F_\psi(x_t)_{\texttt{[REWARD]}}$. Next, a lightweight reward prediction head $g_\phi$ is  applied to estimate the success probability:
\begin{equation}
\hat p_\theta(x_t)=\sigma(g_\phi(h_t)),
\end{equation}
where $\sigma(\cdot)$ denotes the sigmoid function to map the prediction into $[0,1]$.

\noindent\textbf{Training recipe.}
Motivated by Theorem~\ref{theorem:bootstrap-consistency}, we train eVTA$_0$ with a simple temporal-difference (TD) style learning recipe, where a slowly updated target eVTA$_0$ $\hat{p}_{\theta^-}$ parameterized by $\theta^-$ is introduced to generate bootstrapped targets for non-terminal transitions.
Specifically, we construct the training dataset $\mathcal D$ from policy-generated rollouts:
\begin{equation}
\mathcal D=
\{(x_t,x_t^+,d_t,z_t)\},
\end{equation}
where $d_t\in\{0,1\}$ indicates whether timestep $t$ is terminal and $z_t\in\{0,1\}$ denotes the terminal success outcome. 
Here, $x_t$ denotes the input of   eVTA$_0$ $\hat{p}_\theta$, while $x_t^+=[I,\ell,\mathcal O_t^+,\texttt{[REWARD]}]$ represents the input of the target eVTA$_0$ $\hat{p}_{\theta^-}$ , where $\mathcal O_t^+=\{o_{t-w+1},\ldots,o_t,o_{t+1}\}$ extends the  observation history with the next observation $o_{t+1}$. 
This additional observation approximates the one-step-ahead information required by the  bootstrap update. The TD target is then computed as:
\begin{equation}
y_t=d_tz_t+(1-d_t)\hat p_{\theta^-}(x_t^+).
\end{equation}
For terminal transitions, the target directly uses the observed sparse outcome. For non-terminal transitions, the target eVTA$_0$ provides a bootstrapped estimate of the future success probability, enabling terminal supervision to propagate backward along the trajectory.

The parameters of eVTA$_0$ are optimized by minimizing the prediction error between the estimated success probability and the TD target:
\begin{equation}
\mathcal L_{\mathrm{eVTA}}(\theta)
=
\mathbb E_{(x_t,x_t^+,d_t,z_t)\sim\mathcal D}
\left[
(\hat p_\theta(x_t)-y_t)^2
\right].
\end{equation}
The target eVTA$_0$ is updated every $K$ steps using exponential moving average:
\begin{equation}
\theta^-
\leftarrow
\alpha\theta+(1-\alpha)\theta^-,
\end{equation}
where $\alpha$ controls the target update rate.

This TD-style bootstrapping recipe enables eVTA$_0$ to learn dense success-probability feedback from only terminal  outcomes, without  expert demonstrations or  intermediate progress annotations. The detailed training procedure is summarized in Algorithm~\ref{alg:evta-training} of Appendix~\ref{appendix:evta-training}.

\noindent\textbf{Training data.}
The training data for eVTA$_0$ is collected by simply rolling out a robot policy in the environment, where both successful and failed trajectories are retained to form a mixed-quality dataset. Each trajectory only provides a terminal outcome signal, with $z=1$ for successful trajectories and $z=0$ for failed ones. To better utilize failed rollouts, we also adopt an augmentation strategy to assign soft terminal targets for failed trajectories. Details are provided in Appendix~\ref{appendix:evta-training}.

\subsection{Integration into RL Training}

\begin{figure}
    \centering
    \includegraphics[width=0.95\textwidth]{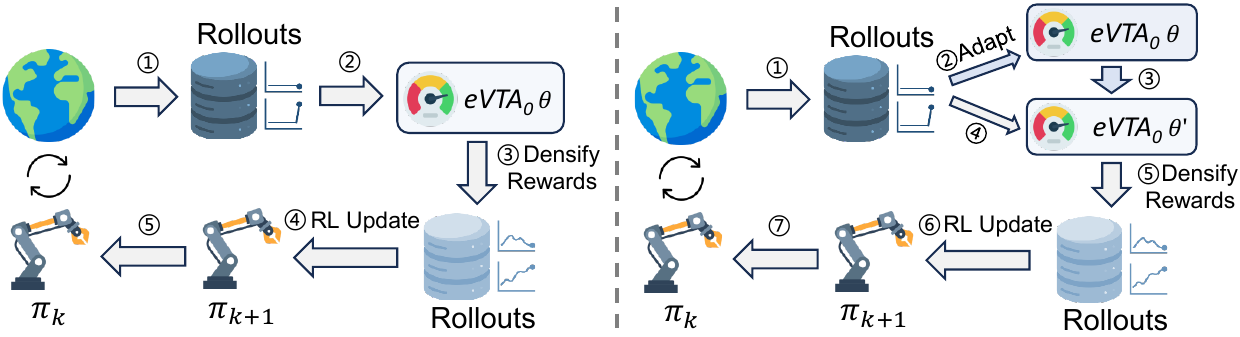}
    \vspace{-0.3cm}
    \caption{
        Illustration of RL training with a fixed eVTA$_0$ (left) and the proposed RL with Evolving Rewards  (RLER) framework (right). RLER  adapts eVTA$_0$ using newly collected policy rollouts, enabling reward learning and policy optimization to  evolve together.
        }
    \label{fig:evta_rl}
\end{figure}

\noindent\textbf{Reward transformation.} Given the success probability $\hat p_t=\hat p_\theta(x_t)$ predicted by eVTA$_0$, a straightforward choice is to directly use $\hat p_t$ as the dense reward. However, since $\hat p_t\in[0,1]$ is non-negative, directly using it as a per-step reward may unintentionally favor longer trajectories, because additional rollout timesteps continue to contribute positive rewards even when the predicted success probability does not improve. To avoid this issue, we convert the predicted success probability into a dense reward signal as:
\begin{equation}
\label{eq:reward-transformation}
r_t^{\mathrm{eVTA}}
=
-(1-\hat p_t)
=
\hat p_t-1.
\end{equation}
Since $1-\hat p_t$ corresponds to the predicted probability of eventual task failure, $r_t^{\mathrm{eVTA}}$ can be naturally interpreted as a failure-probability penalty. States associated with higher probabilities of eventual failure receive larger penalties, whereas the penalty approaches zero as the predicted success probability approaches one. Notably, since this transformation only introduces a constant shift, it preserves the ordering and pairwise differences of eVTA$_0$ predictions. With this transformation, a trained eVTA$_0$ can be directly integrated into standard RL training loops, including online RL (e.g., PPO~\citep{schulman2017proximal} and GRPO~\citep{guo2025deepseek}) and offline RL (e.g., IQL~\citep{kostrikov2021offline}, RECAP~\citep{intelligence2025pi06}).


\noindent\textbf{RL with Evolving Rewards  (RLER).}
The ability of eVTA$_0$ to learn directly from policy-generated rollouts enables reward learning to adapt alongside policy optimization. We therefore introduce RLER, a closed-loop framework in which reward modeling and policy optimization  evolve together. As shown in Figure~\ref{fig:evta_rl}, new rollouts generated by the policy  are periodically used to update eVTA$_0$. The updated eVTA$_0$ then provides success-probability feedback for subsequent policy optimization. As the  behavior distribution of the policy  changes,  updating eVTA$_0$ allows it to capture newly emerging behaviors and provide more informative guidance for future policy updates. 

\section{Evaluation}


We conduct extensive experiments to evaluate the effectiveness of eVTA$_0$  for embodied policy learning, with the goal of answering the following key questions. \textbf{(Q1) Reward quality:} How does eVTA$_0$ compare with existing reward models in terms of reward quality? \textbf{(Q2) Policy learning:} How effectively does eVTA$_0$ facilitate policy learning compared with existing methods? \textbf{(Q3) Reward evolution:} Can eVTA$_0$  adapt to an evolving policy and further improve policy optimization through RLER? More experiments and analysis are provided  in  Appendix~\ref{appendix:additional-results}.

\begin{table}[t]
   \centering
   \renewcommand{\arraystretch}{0.875}
   \caption{VOC and VROC results on LIBERO and MetaWorld. The best and second-best average results of each benchmark are shown in bold and underlined, respectively.}
   \label{tab:voc-vroc}
   \label{tab:libero-voc-vroc}
   \label{tab:metaworld-voc-vroc}
   \setlength{\tabcolsep}{3pt}
   \resizebox{\textwidth}{!}{%
   \begin{tabular}{@{}lNNNNANNNNA@{}}
       \toprule
       \multicolumn{11}{c}{\textbf{LIBERO}} \\
       \midrule
       & \multicolumn{5}{c}{\textbf{VOC} ($\uparrow$ Better)} & \multicolumn{5}{c}{\textbf{VROC} ($\uparrow$ Better)} \\
       \cmidrule(lr){2-6}\cmidrule(l){7-11}
       \textbf{Method}  & Spatial & Object & Goal & Long & Avg. & Spatial & Object & Goal & Long & Avg. \\
       \midrule
       GVL (Qwen3-VL-4B)              & 0.14 & 0.59 & 0.12 & 0.28 & 0.28 & 0.00 & 0.16 & 0.04 & 0.03 & 0.06 \\
       GVL (Qwen3-VL-8B)              & 0.25 & 0.46 & 0.33 & 0.37 & 0.35 & -0.01 & -0.04 & -0.03 & 0.10 & 0.00 \\
       GVL (GPT-5)                & 0.79 & 0.79 & 0.72 & 0.66 & 0.74 & 0.30 & 0.53 & 0.52 & 0.57 & 0.48 \\
       TOPReward (Qwen3-VL-4B)        & 0.70 & 0.73 & 0.74 & 0.68 & 0.71 & -0.73 & -0.33 & -0.71 & -0.55 & -0.58 \\
       TOPReward (Qwen3-VL-8B)        & 0.65 & 0.85 & 0.76 & 0.79 & 0.76 & -0.59 & -0.13 & -0.30 & -0.34 & -0.34 \\
       VLAC-2B                    & -0.01 & 0.27 & 0.00 & 0.01 & 0.07 & -0.08 & 0.03 & -0.04 & -0.08 & -0.04 \\
       VLAC-2B (Fine-tuned)        & 0.59 & 0.60 & 0.36 & 0.37 & 0.48 & 0.61 & 0.63 & 0.35 & 0.36 & 0.49 \\
       VLAC-8B                    & -0.06 & 0.54 & -0.13 & -0.04 & 0.08 & -0.11 & 0.52 & -0.23 & -0.03 & 0.04 \\
       VLAC-8B (Fine-tuned)        & 0.67 & 0.65 & 0.41 & 0.39 & 0.53 & 0.59 & 0.58 & 0.35 & 0.37 & 0.47 \\
       Robometer-4B               & 0.03 & 0.01 & -0.04 & -0.04 & -0.01 & 0.05 & 0.08 & 0.29 & -0.06 & 0.09 \\
       Robometer-4B (Fine-tuned)   & 0.97 & 0.98 & 0.97 & 0.98 & \textbf{0.98} & 0.22 & 0.67 & 0.80 & 0.77 & \underline{0.61} \\
       \rowcolor{evtarow} \textbf{eVTA$_0$} (Qwen3-VL-4B) & 0.87 & 0.90 & 0.84 & 0.77 & \underline{0.84} & 0.73 & 0.88 & 0.74 & 0.73 & \textbf{0.77} \\
       \midrule
       \multicolumn{11}{c}{\textbf{MetaWorld}} \\
       \midrule
       & \multicolumn{5}{c}{\textbf{VOC} ($\uparrow$ Better)} & \multicolumn{5}{c}{\textbf{VROC} ($\uparrow$ Better)} \\
       \cmidrule(lr){2-6}\cmidrule(l){7-11}
       \textbf{Method}  & Easy & \makebox[\linewidth][c]{Medium} & Hard & \makebox[\linewidth][c]{Very Hard} & Avg. & Easy & \makebox[\linewidth][c]{Medium} & Hard & \makebox[\linewidth][c]{Very Hard} & Avg. \\
       \midrule
       GVL (Qwen3-VL-4B)              & 0.34 & 0.17 & 0.04 & 0.18 & 0.18 & -0.09 & -0.09 & 0.04 & 0.11 & -0.01 \\
       GVL (Qwen3-VL-8B)              & 0.41 & 0.22 & 0.14 & 0.16 & 0.23 & -0.23 & -0.20 & -0.05 & -0.05 & -0.13 \\
       GVL (GPT-5)                & 0.82 & 0.71 & 0.40 & 0.61 & 0.64 & 0.65 & 0.59 & -0.02 & 0.18 & 0.35 \\
       TOPReward (Qwen3-VL-4B)        & 0.78 & 0.78 & 0.79 & 0.83 & \underline{0.79} & -0.81 & -0.82 & -0.75 & -0.84 & -0.81 \\
       TOPReward (Qwen3-VL-8B)        & 0.86 & 0.58 & 0.44 & 0.61 & 0.62 & -0.65 & -0.64 & -0.69 & -0.70 & -0.67 \\
       VLAC-2B                    & -0.01 & 0.06 & -0.02 & 0.08 & 0.03 & 0.07 & 0.04 & -0.02 & 0.08 & 0.04 \\
       VLAC-2B (Fine-tuned)        & 0.24 & 0.13 & 0.29 & 0.36 & 0.25 & 0.41 & 0.27 & 0.46 & 0.47 & 0.40 \\
       VLAC-8B                    & 0.02 & -0.02 & 0.07 & -0.01 & 0.01 & -0.07 & -0.05 & -0.05 & 0.08 & -0.02 \\
       VLAC-8B (Fine-tuned)        & 0.41 & 0.27 & 0.36 & 0.63 & 0.42 & 0.48 & 0.24 & 0.41 & 0.57 & \underline{0.43} \\
       Robometer-4B               & -0.03 & -0.24 & 0.23 & 0.04 & 0.00 & -0.06 & -0.01 & 0.05 & 0.00 & -0.01 \\
       Robometer-4B (Fine-tuned)   & 0.90 & 0.83 & 0.62 & 0.71 & 0.77 & 0.68 & 0.07 & 0.13 & -0.27 & 0.15 \\
       \rowcolor{evtarow} \textbf{eVTA$_0$} (Qwen3-VL-4B) & 0.79 & 0.81 & 0.82 & 0.80 & \textbf{0.80} & 0.38 & 0.52 & 0.42 & 0.43 & \textbf{0.44} \\
       \bottomrule
   \end{tabular}%
   }
\end{table}

\subsection{Reward Quality}




\noindent\textbf{Setup.}
We conduct experiments on two representative benchmarks LIBERO~\citep{fei2025libero} and MetaWorld~\citep{yu2020meta}, which provide diverse task suites for comprehensive evaluation. We compare eVTA$_0$ against state-of-the-art general-purpose reward models, including training-free methods GVL~\citep{ma2024vision} and TOPReward~\citep{chen2026topreward}, and pretrained reward models VLAC~\citep{zhai2025vision} and Robometer~\citep{liang2026robometer}. 
For pretrained reward models, we evaluate both the original VLAC and Robometer models and their variants  fine-tuned on the downstream benchmarks. 
We systematically evaluate the reward quality of each method from two complementary perspectives: \textit{temporal consistency} and \textit{outcome distinction}. Temporal consistency is measured on expert demonstrations using VOC~\citep{ma2024vision} and VROC~\citep{zhai2025vision}, which assess whether rewards consistently reflect progress toward and deviation from task completion, respectively. Outcome distinction is evaluated on mixed-quality policy rollouts using MSE and Kendall's $\tau_a$~\citep{liang2026robometer}, measuring terminal success estimation and discrimination between successful and failed trajectories, respectively. Additional details on the experimental setup, baselines and eVTA$_0$ implementation are provided in  Appendix~\ref{appendix:reward-quality-evaluation}.

\noindent\textbf{Results.}
Tables~\ref{tab:voc-vroc} and~\ref{tab:libero-mse-kendall} report the performance of each method. For temporal consistency, training-free methods generally achieve   higher VOC than VROC, indicating that they can better capture forward task progression than deviations  from successful completion. A similar pattern is observed for pretrained reward models. Although downstream fine-tuning substantially improves their performance, their VROC remains much lower than VOC. This  may stem from their strong reliance on expert demonstrations, which provide limited information about how policy rollouts deviate from successful completion. Consequently, the learned rewards tend to  provide less informative feedback on failed rollouts, which may limit their ability to guide policies to learn  from failures. In comparison, eVTA$_0$, despite being trained solely on mixed-quality policy rollouts without expert demonstrations, performs consistently well in both metrics. It achieves the best average VROC on both LIBERO and MetaWorld, while achieving the second-best and best average VOC, respectively. More importantly, when evaluated directly on mixed-quality policy rollouts, eVTA$_0$ also achieves the best average MSE and Kendall's $\tau_a$. These results show that eVTA$_0$ better distinguishes successful and failed trajectories while providing more accurate estimates of terminal success likelihood. This capability is particularly important for RL policy learning, where collected trajectories naturally contain a mixture of successful, suboptimal, and failed rollouts.

\begin{table}[t]
    \centering
    \renewcommand{\arraystretch}{0.875}
    \caption{MSE and Kendall's $\tau_a$ results on LIBERO and MetaWorld. The best and second-best average results of each benchmark are shown in bold and underlined, respectively.}
    \label{tab:libero-mse-kendall}
    \label{tab:metaworld-mse-kendall}
    \setlength{\tabcolsep}{3pt}
    \resizebox{\textwidth}{!}{%
    \begin{tabular}{@{}lNNNNANNNNA@{}}
        \toprule
        \multicolumn{11}{c}{\textbf{LIBERO}} \\
        \midrule
        & \multicolumn{5}{c}{\textbf{MSE} ($\downarrow$ Better)} & \multicolumn{5}{c}{\textbf{Kendall's} $\tau_a$ ($\uparrow$ Better)} \\
        \cmidrule(lr){2-6}\cmidrule(l){7-11}
        \textbf{Method} & Spatial & Object & Goal & Long & Avg. & Spatial & Object & Goal & Long & Avg. \\
        \midrule
        GVL (Qwen3-VL-4B)              & 0.66 & 0.54 & 0.39 & 0.44 & 0.51 & 0.00 & 0.73 & 0.13 & 0.00 & 0.21 \\
        GVL (Qwen3-VL-8B)              & 0.72 & 0.76 & 0.50 & 0.43 & 0.60 & -0.17 & 0.00 & 0.18 & 0.02 & 0.01 \\
        GVL (GPT-5)                & 0.64 & 0.49 & 0.40 & 0.40 & 0.48 & 0.37 & 0.60 & 0.30 & -0.05 & 0.31 \\
        TOPReward (Qwen3-VL-4B)        & 0.73 & 0.76 & 0.53 & 0.56 & 0.64 & -0.27 & 0.62 & 0.21 & 0.25 & 0.20 \\
        TOPReward (Qwen3-VL-8B)        & 0.73 & 0.76 & 0.53 & 0.56 & 0.64 & -0.13 & 0.33 & 0.34 & 0.33 & 0.22 \\
        VLAC-2B                    & 0.72 & 0.71 & 0.52 & 0.56 & 0.63 & 0.00 & 0.08 & 0.02 & -0.02 & 0.02 \\
        VLAC-2B (Fine-tuned)        & 0.62 & 0.73 & 0.47 & 0.51 & 0.58 & 0.49 & 0.10 & 0.21 & 0.07 & 0.22 \\
        VLAC-8B                    & 0.72 & 0.72 & 0.51 & 0.55 & 0.63 & -0.10 & 0.10 & 0.05 & -0.08 & -0.01 \\
        VLAC-8B (Fine-tuned)        & 0.53 & 0.70 & 0.44 & 0.51 & 0.54 & 0.51 & 0.13 & 0.28 & 0.00 & 0.23 \\
        Robometer-4B               & 0.25 & 0.22 & 0.19 & 0.17 & 0.21 & 0.22 & 0.09 & -0.02 & -0.16 & 0.03 \\
        Robometer-4B (Fine-tuned)   & 0.07 & 0.09 & 0.13 & 0.10 & \underline{0.10} & 0.77 & 0.77 & 0.96 & 0.90 & \underline{0.85} \\
        \rowcolor{evtarow} \textbf{eVTA$_0$} (Qwen3-VL-4B) & 0.06 & 0.03 & 0.07 & 0.08 & \textbf{0.06} & 0.98 & 1.00 & 1.00 & 0.93 & \textbf{0.98} \\
        \midrule
        \multicolumn{11}{c}{\textbf{MetaWorld}} \\
        \midrule
        & \multicolumn{5}{c}{\textbf{MSE} ($\downarrow$ Better)} & \multicolumn{5}{c}{\textbf{Kendall's} $\tau_a$ ($\uparrow$ Better)} \\
        \cmidrule(lr){2-6}\cmidrule(l){7-11}
        \textbf{Method} & Easy & \makebox[\linewidth][c]{Medium} & Hard & \makebox[\linewidth][c]{Very Hard} & Avg. & Easy & \makebox[\linewidth][c]{Medium} & Hard & \makebox[\linewidth][c]{Very Hard} & Avg. \\
        \midrule
        GVL (Qwen3-VL-4B)              & 0.34 & 0.48 & 0.39 & 0.47 & 0.42 & 0.31 & -0.12 & 0.19 & -0.31 & 0.02 \\
        GVL (Qwen3-VL-8B)              & 0.46 & 0.51 & 0.39 & 0.50 & 0.47 & 0.30 & -0.16 & 0.43 & -0.26 & 0.08 \\
        GVL (GPT-5)                & 0.22 & 0.32 & 0.26 & 0.36 & 0.29 & 0.80 & 0.02 & 0.57 & 0.15 & 0.39 \\
        TOPReward (Qwen3-VL-4B)        & 0.63 & 0.68 & 0.60 & 0.62 & 0.63 & 0.16 & -0.10 & 0.45 & 0.07 & 0.14 \\
        TOPReward (Qwen3-VL-8B)        & 0.63 & 0.68 & 0.60 & 0.62 & 0.63 & 0.73 & 0.19 & 0.82 & 0.53 & 0.57 \\
        VLAC-2B                    & 0.61 & 0.68 & 0.60 & 0.62 & 0.63 & 0.22 & 0.00 & 0.07 & 0.02 & 0.08 \\
        VLAC-2B (Fine-tuned)        & 0.60 & 0.62 & 0.59 & 0.60 & 0.60 & 0.47 & 0.41 & 0.43 & 0.66 & 0.49 \\
        VLAC-8B                    & 0.63 & 0.68 & 0.60 & 0.62 & 0.63 & 0.02 & 0.00 & 0.07 & 0.14 & 0.06 \\
        VLAC-8B (Fine-tuned)        & 0.61 & 0.61 & 0.56 & 0.57 & 0.59 & 0.38 & 0.40 & 0.47 & 0.41 & 0.42 \\
        Robometer-4B               & 0.17 & 0.23 & 0.14 & 0.26 & 0.20 & 0.73 & 0.30 & 0.78 & 0.05 & 0.47 \\
        Robometer-4B (Fine-tuned)   & 0.08 & 0.08 & 0.14 & 0.22 & \underline{0.13} & 0.80 & 0.63 & 0.88 & 0.30 & \underline{0.65} \\
        \rowcolor{evtarow} \textbf{eVTA$_0$} (Qwen3-VL-4B) & 0.02 & 0.05 & 0.03 & 0.06 & \textbf{0.04} & 1.00 & 0.84 & 1.00 & 0.93 & \textbf{0.94} \\
        \bottomrule
    \end{tabular}%
    }
 \end{table}
\begin{figure}[t]
    \centering
    \includegraphics[width=0.99\textwidth]{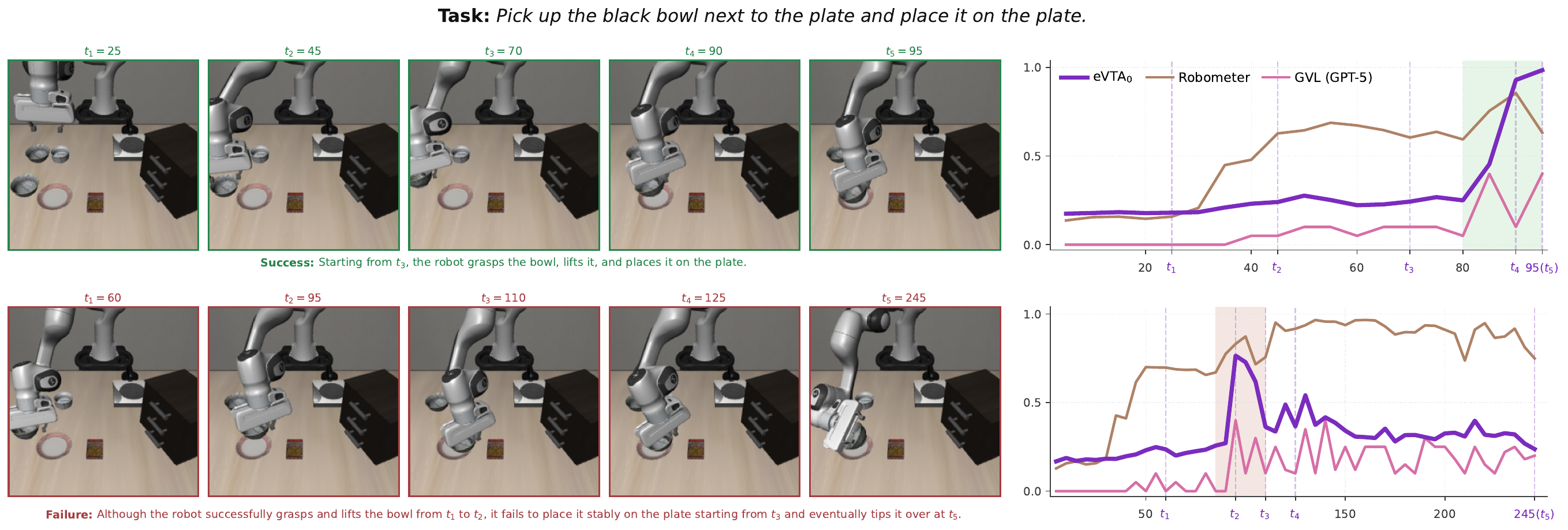}
    \vspace{-0.3cm}
    \caption{
        {Qualitative comparison on successful and failed trajectories under the same  task.}
        }
    \label{fig:visualization}
    \vspace{-0.3cm}
 \end{figure}

\noindent\textbf{Qualitative comparison.}
Figure~\ref{fig:visualization} further visualizes the predictions of eVTA$_0$, Robometer, and GVL (GPT-5) for two trajectories of different outcomes under the same manipulation task. 
For the successful trajectory, eVTA$_0$ remains conservative during early manipulation and sharply increases its predicted success probability once the bowl is reliably moved toward and placed on the plate, eventually approaching one at task completion. For the failed trajectory, eVTA$_0$ initially increases its prediction when the robot successfully grasps and lifts the bowl, but decreases after the placement becomes unstable and remains low when the bowl is eventually tipped over. In contrast, Robometer continues to assign high scores even after the failed trajectory deviates from successful completion, while GVL produces substantially noisier predictions.
These results show that eVTA$_0$ better tracks the actual evolution of task success throughout policy rollouts.

\subsection{Policy Learning}

\begin{figure}[t]
   \centering
   \includegraphics[width=0.99\textwidth]{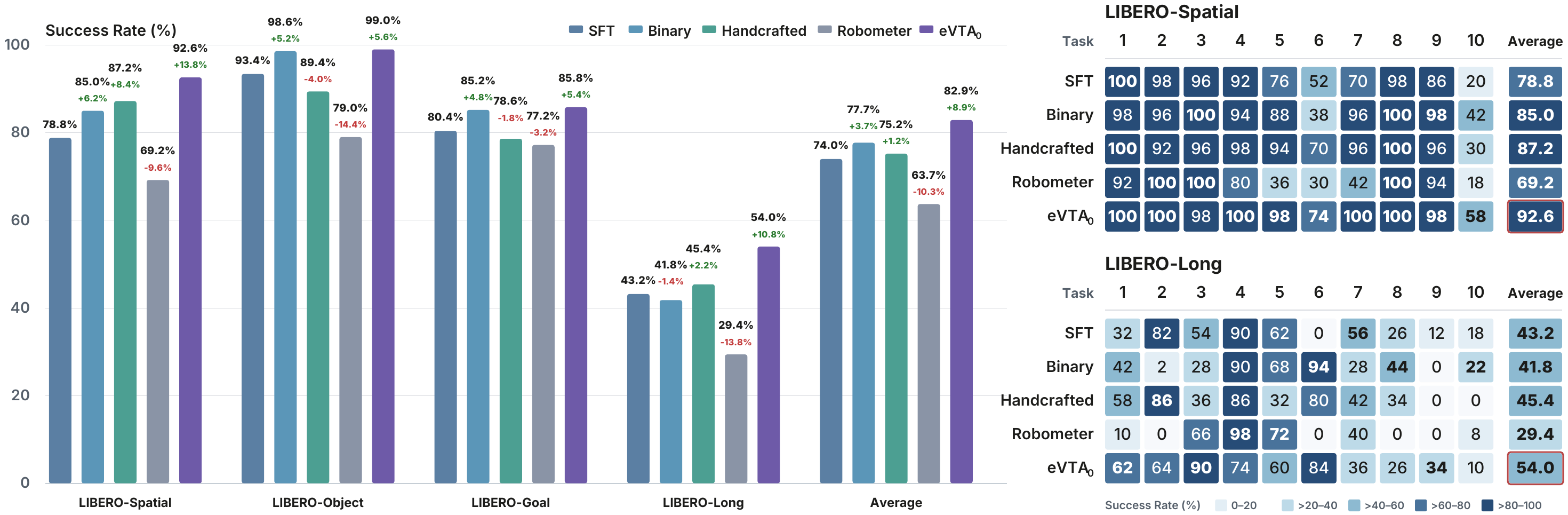}
   \vspace{-0.2cm}
\caption{Policy performance with different rewards under the same RL training budget.}
   \label{fig:rq2_policy_performance}
   \vspace{-0.3cm}
\end{figure}

\begin{figure}[t]
    \centering
    \includegraphics[width=0.99\textwidth]{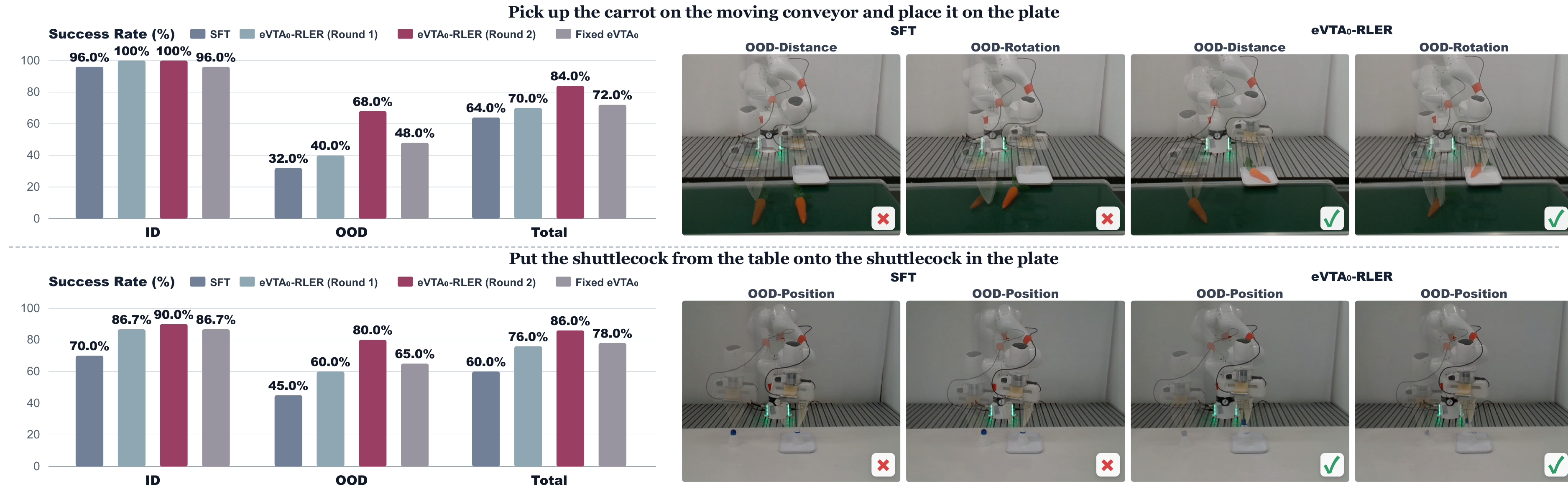}
    \vspace{-0.2cm}
 \caption{Real-world success rates across RLER rounds compared with SFT and fixed eVTA$_0$.}
    \label{fig:real_world}
    \vspace{-0.3cm}
 \end{figure}

\noindent\textbf{Setup.}
We evaluate whether eVTA$_0$ can effectively improve RL-based policy learning on LIBERO task suites, covering diverse manipulation settings and long-horizon tasks. We use $\pi_{0.5}$~\citep{intelligence2025pi} as the base policy\footnote{We use the $\pi_{0.5}$ checkpoint from RLinf~\citep{yu2026rlinf} (\url{https://huggingface.co/RLinf/RLinf-Pi05-LIBERO-SFT}), fine-tuned on LIBERO with 40 demonstrations in total across all tasks.} and GRPO as the RL algorithm, and compare eVTA$_0$ with binary reward, the handcrafted reward used in $\pi_{0.6}^\star$~\citep{intelligence2025pi06}, and Robometer. The SFT policy before RL is also reported as a reference. 
All methods are trained under the same budget of 6,400 interaction episodes  with identical RL configurations, allowing us to compare  policy improvement and sample efficiency under a controlled setting. 
More details are provided in  Appendix~\ref{appendix:policy-learning-evaluation}.


\noindent\textbf{Results.}
Figure~\ref{fig:rq2_policy_performance} reports the policy performance after RL training with different reward signals. eVTA$_0$ achieves the best average success rate on all task suites, corresponding to absolute improvements of $13.8\%$, $5.6\%$, $5.4\%$, and $10.8\%$ on LIBERO-Spatial, LIBERO-Object, LIBERO-Goal, and LIBERO-Long over the SFT policy, respectively. In particular, the advantage is more evident on  more challenging, long-horizon task suite LIBERO-Long. Handcrafted rewards yield only $+2.2\%$ changes over SFT, while binary rewards and Robometer result in $1.4\%$ and $13.8\%$ degradations, respectively. In contrast, eVTA$_0$ improves the success rate by $10.8\%$. The task-wise results also show that eVTA$_0$ maintains a more balanced performance across different tasks, whereas other rewards exhibit large performance variations across tasks. The relatively weak RL performance of Robometer is also consistent with its less favorable reward-quality results on mixed-quality rollouts, suggesting that reward quality on policy-generated data can be important for effective policy learning. 
These results show that eVTA$_0$ provides consistently effective reward guidance across diverse  settings and enables more sample-efficient  policy learning under the same training budget.

\subsection{Reward Evolution}
\noindent\textbf{Setup.}
We evaluate whether reward learning can adapt alongside policy optimization through RLER in real-world manipulation tasks. Starting from the SFT policy, RLER alternates between adapting eVTA$_0$ with newly collected rollouts and optimizing the policy using IQL. We compare against a fixed eVTA$_0$ trained on the same accumulated rollout data, with matched data and policy-training budgets within each task. We report both in-distribution (ID) and out-of-distribution (OOD) performance. More details are available  in Appendix~\ref{appendix:real-world-evaluation}.

\noindent\textbf{Results.}
Figure~\ref{fig:real_world} shows that RLER consistently improves policy performance over successive learning rounds. Across real-world tasks, RLER improves the overall success rate by $20\%$--$26\%$ over the initial SFT policy and outperforms the fixed eVTA$_0$ baseline by $8\%$--$12\%$, with particularly clear gains under OOD conditions. Notably, these improvements are achieved under matched data and policy-training budgets, suggesting that the benefit comes from adapting reward feedback as the policy evolves. 
Conceptually, RLER transforms eVTA$_0$ from a static evaluator into an adaptive teacher that follows the policy's learning frontier: as the policy encounters new behaviors and failure modes, its own rollouts refine the reward feedback, allowing the learning signal  to evolve with the learner. This property is especially valuable for real-world RL, where reliable reward models are difficult to obtain beforehand and early policy rollouts are often imperfect. RLER enables reward learning to start from such imperfect experience and progressively improve alongside the policy.

\section{Conclusion}
In this work, we explore  a new paradigm for reward learning where reward feedback can be learned directly from policy rollouts rather than relying on expert demonstrations.
We show that sparse terminal outcomes implicitly define dense success-probability feedback, and build eVTA$_0$ to learn such feedback directly from mixed-quality policy rollouts without expert demonstrations or intermediate  labels. Furthermore, we  introduce RLER, which allows reward feedback to adapt alongside policy optimization. Experiments across simulation and real-world manipulation tasks show that eVTA$_0$ provides more reliable reward signals and improves policy learning, while RLER further enables reward learning to evolve with the policy. We hope this work provides a simple yet effective foundation toward scalable reward learning for generalist robot policies, where rewards can be learned  from experience and adapt as the policy evolves. 

\section*{AI Use Statement}
In this work, we used generative AI tools for assisting in the writing of proofs, providing feedback on research experiments, and assisting with method implementation. We have not used generative AI tools for other tasks with required disclosure, and the rest of the required disclosure tasks are not applicable to this work. Additionally, we used generative AI tools for suggesting experimental parameters, polishing the research paper to improve readability, and creating or modifying experimental figures or tables. We have reviewed all AI-assisted work. AI-assisted proofs and experimental suggestions were carefully examined by the authors, AI-assisted code was verified and tested for correctness, and AI-assisted figures and tables were manually checked against the corresponding experimental results. We take responsibility for the final content of this work, including text, claims or artifacts produced with the aid of generative AI.

\section*{Ethics Statement}

This work does not involve human subjects, personally identifiable information, or other ethical concerns requiring specific discussion.


\section*{Reproducibility Statement}

We provide the methodological and experimental details necessary to facilitate reproduction of our results. The formulation of eVTA$_0$, its TD-style training procedure, and its integration into RL are described in Section~\ref{sec:method}. Complete proofs of the theoretical results are provided in Appendix~\ref{appendix:math-proof}. Appendix~\ref{appendix:evta-training} further details the design and training of eVTA$_0$, including the instruction prompt, training algorithm, and dataset augmentation strategy. Appendix~\ref{appendix:experimental-setup} describes the experimental setups, including datasets, baselines, eVTA$_0$ implementation, model configurations, training hyperparameters, and evaluation settings. Additional ablation studies, qualitative analyses, and detailed experimental results are reported in Appendix~\ref{appendix:additional-results}. Appendix~\ref{appendix:discussion} provides additional discussion and analysis of the design and properties of eVTA$_0$. We also plan to release the source code, the training and evaluation data generated in this work, and the trained model checkpoints to further facilitate reproducibility.

\bibliography{main}
\bibliographystyle{iclr2027_conference}

\clearpage
\appendix
\startcontents[appendices]
\begingroup
\setcounter{tocdepth}{2}
\section*{Appendix Contents}
\printcontents[appendices]{}{1}{}
\endgroup
\clearpage
\section{Proofs of Theorems}
\label{appendix:math-proof}

Recall that $S\in\{0,1\}$ denotes the terminal success label of a rollout of length $T$, $\mathcal F_t=\sigma(H_t)$ is the information revealed by the trajectory prefix up to timestep $t$, and $r_t=\mathbb E_\pi[S\mid\mathcal F_t]$ is the success-probability reward of Eq.~\ref{eq:success-probability-reward}. We provide the detailed proofs of each theorem as below. 

\subsection{Proof of Theorem~\ref{theorem:martingale}}
\label{appendix:math-proof-martingale}

\setcounter{theorem}{0}%
\begingroup
\renewcommand{\theHtheorem}{restated.\thetheorem}%
\begin{theorem}[Temporal consistency of success-probability rewards]
\label{theorem:martingale-appendix}
The success probability reward defined above satisfies
$r_t=\mathbb E_\pi[r_T\mid\mathcal F_t]$ for every $t\le T$.
Moreover, for every $t<T$, it satisfies the temporal consistency property:
\begingroup
\renewcommand{\theHequation}{restated.eq.martingale}%
\begin{equation}
\label{eq:martingale-appendix}
\tag{\ref{eq:martingale}}
\mathbb E_\pi[r_{t+1}\mid\mathcal F_t]=r_t.
\end{equation}
\endgroup
\end{theorem}
\endgroup

\begin{proof}
Since $r_T=S$ and, by definition, $r_t=\mathbb E_\pi[S\mid\mathcal F_t]$
(Eq.~\ref{eq:success-probability-reward}), we have
\begin{equation*}
\mathbb E_\pi[r_T\mid\mathcal F_t]
=
\mathbb E_\pi[S\mid\mathcal F_t]
=
r_t,
\end{equation*}
which proves the first claim.

For the second claim, recall the tower property of conditional expectation
\citep{williams1991probability}: if $\mathcal G_1\subseteq\mathcal G_2$ and $X$
is integrable, then:
\begin{equation*}
\mathbb E_\pi
\left[
\mathbb E_\pi[X\mid\mathcal G_2]
\mid
\mathcal G_1
\right]
=
\mathbb E_\pi[X\mid\mathcal G_1].
\end{equation*}
Applying this with $X=S$, $\mathcal G_1=\mathcal F_t$, and
$\mathcal G_2=\mathcal F_{t+1}$ (admissible since
$\mathcal F_t\subseteq\mathcal F_{t+1}$), together with the definition of
$r_{t+1}$, gives:
\begin{equation*}
\mathbb E_\pi[r_{t+1}\mid\mathcal F_t]
=
\mathbb E_\pi
\left[
\mathbb E_\pi[S\mid\mathcal F_{t+1}]
\mid
\mathcal F_t
\right]
=
\mathbb E_\pi[S\mid\mathcal F_t]
=
r_t,
\end{equation*}
which completes the proof of Eq.~\ref{eq:martingale-appendix}.
\end{proof}

\subsection{Proof of Theorem~\ref{theorem:bootstrap-consistency}}
\label{appendix:math-proof-consistency}

\begingroup
\renewcommand{\theHtheorem}{restated.\thetheorem}%
\begin{theorem}[Population consistency of one-step bootstrapping]
\label{theorem:bootstrap-consistency-appendix}
Consider the ideal backward regression defined by the terminal anchor
$p_T^\star:=S$ and, for each $t<T$:
\begingroup
\renewcommand{\theHequation}{restated.eq.objective}%
\begin{equation}
\label{eq:one-step-objective-appendix}
\tag{\ref{eq:one-step-objective}}
p_t^\star
\in
\arg\min_{g\in L_2(\mathcal F_t)}
\mathbb E_\pi\left[\left(g-p_{t+1}^\star\right)^2\right],
\end{equation}
\endgroup
where $L_2(\mathcal F_t)$ denotes the square-integrable
$\mathcal F_t$-measurable candidate predictions. Then the population minimizer
is unique up to almost-sure equality and satisfies, for every $t\le T$:
\begingroup
\renewcommand{\theHequation}{restated.eq.optimal}%
\begin{equation}
\label{eq:one-step-optimal-appendix}
\tag{\ref{eq:one-step-optimal}}
p_t^\star=r_t=\mathbb E_\pi[S\mid\mathcal F_t].
\end{equation}
\endgroup
Thus, terminal supervision can be propagated backward through ideal one-step
regression without changing the population success-probability target.
\end{theorem}
\endgroup

\begin{proof}
Because $S\in\{0,1\}$, each
$r_t=\mathbb E_\pi[S\mid\mathcal F_t]$ is bounded and therefore belongs to
$L_2(\mathcal F_t)$. Hence all of the regression risks below are finite. We
prove the claim by backward induction on $t$.

\textbf{Base case ($t=T$).} By definition, $p_T^\star=S=r_T$ almost surely.

\textbf{Induction step.} Fix $t<T$ and assume the induction hypothesis that
$p_{t+1}^\star=r_{t+1}$ almost surely. Define:
\begin{equation*}
m_t:=\mathbb E_\pi[p_{t+1}^\star\mid\mathcal F_t].
\end{equation*}
The variable $m_t$ is $\mathcal F_t$-measurable and lies in
$L_2(\mathcal F_t)$. For any $g\in L_2(\mathcal F_t)$, since
$g-p_{t+1}^\star=(g-m_t)+(m_t-p_{t+1}^\star)$, expanding the square gives:
\begin{equation*}
\begin{aligned}
\mathbb E_\pi\!\left[\left(g-p_{t+1}^\star\right)^2\right]
={}&\mathbb E_\pi\!\left[(g-m_t)^2\right]
+\mathbb E_\pi\!\left[(p_{t+1}^\star-m_t)^2\right] \\
&+2\mathbb E_\pi\!\left[(g-m_t)(m_t-p_{t+1}^\star)\right].
\end{aligned}
\end{equation*}
For the cross term, we first evaluate the conditional expectation of its
second factor. By the linearity of conditional expectation, the
$\mathcal F_t$-measurability of $m_t$, and the definition of $m_t$,
\begin{equation*}
\mathbb E_\pi[m_t-p_{t+1}^\star\mid\mathcal F_t]
=\mathbb E_\pi[m_t\mid\mathcal F_t]
-\mathbb E_\pi[p_{t+1}^\star\mid\mathcal F_t]
=m_t-m_t
=0.
\end{equation*}
Moreover, since $g-m_t$ is $\mathcal F_t$-measurable, conditioning the whole
product on $\mathcal F_t$ and pulling out the known factor $g-m_t$ (tower
property; \citealp{williams1991probability}) gives:
\begin{equation*}
\mathbb E_\pi\!\left[(g-m_t)(m_t-p_{t+1}^\star)\right]
=
\mathbb E_\pi\!\left[(g-m_t)
\mathbb E_\pi[m_t-p_{t+1}^\star\mid\mathcal F_t]\right]
=0.
\end{equation*}
Consequently,
\begin{equation*}
\begin{aligned}
\mathbb E_\pi\!\left[\left(g-p_{t+1}^\star\right)^2\right]
={}&\mathbb E_\pi\!\left[(g-m_t)^2\right]
+\mathbb E_\pi\!\left[
\operatorname{Var}_\pi(p_{t+1}^\star\mid\mathcal F_t)
\right].
\end{aligned}
\end{equation*}
The second term does not depend on $g$, whereas the first term is nonnegative
and equals zero if and only if $g=m_t$ almost surely. Therefore the minimizer
over $L_2(\mathcal F_t)$, unique up to almost-sure equality, is $m_t$.
Finally, the definition of $m_t$, the induction hypothesis, and the temporal
consistency of Eq.~\ref{eq:martingale-appendix} give:
\begin{equation*}
m_t
=\mathbb E_\pi[p_{t+1}^\star\mid\mathcal F_t]
=\mathbb E_\pi[r_{t+1}\mid\mathcal F_t]
=r_t.
\end{equation*}
Thus $p_t^\star=r_t$ almost surely, completing the backward induction for all
$t\le T$.
\end{proof}

\section{Additional Details of eVTA$_0$ Design}
\label{appendix:evta-training}

\begin{algorithm}[t]
    \caption{Training eVTA$_0$ via TD-Style Bootstrapping}
    \label{alg:evta-training}
    \begin{algorithmic}[1]
    
    \Require Training dataset $\mathcal{D}$, eVTA$_0$ $\hat p_{\theta}$,
    number of training epochs $E$, target update interval $K$,
    target update rate $\alpha$, learning rate $\eta$, batch size $B$
    \Ensure Trained eVTA$_0$ $\hat p_{\theta}$
    
    \State Initialize target parameters $\theta^- \leftarrow \theta$
    \State Initialize training step $s \leftarrow 0$
    
    \For{$e=1,\ldots,E$}
        \For{each mini-batch $\mathcal{B}\subset\mathcal{D}$ with $|\mathcal{B}|=B$}
    
            \For{each $(x_t,x_t^+,d_t,z_t)\in\mathcal{B}$}
                \If{$d_t=1$}
                    \State $y_t \leftarrow z_t$
                    \Comment{Terminal outcome}
                \Else
                    \State $y_t \leftarrow
                    \operatorname{stopgrad}\!\left(
                    \hat p_{\theta^-}(x_t^+)
                    \right)$
                    \Comment{TD bootstrap}
                \EndIf
            \EndFor
    
            \State Compute the TD regression loss:
            \[
            \mathcal{L}_{\mathrm{eVTA}}(\theta)
            =
            \frac{1}{|\mathcal{B}|}
            \sum_{(x_t,x_t^+,d_t,z_t)\in\mathcal{B}}
            \left(
            \hat p_{\theta}(x_t)-y_t
            \right)^2
            \]
    
            \State Update eVTA$_0$:
            $\theta \leftarrow
            \theta-\eta\nabla_{\theta}
            \mathcal{L}_{\mathrm{eVTA}}(\theta)$
    
            \State $s \leftarrow s+1$
    
            \If{$s \bmod K = 0$}
                \State $\theta^- \leftarrow
                \alpha\theta+(1-\alpha)\theta^-$
                \Comment{Update target eVTA$_0$}
            \EndIf
    
        \EndFor
    \EndFor
    
    \State \Return $\hat p_{\theta}$
    
\end{algorithmic}
\end{algorithm}

\noindent\textbf{Instruction prompt.} 
As shown in Figure~\ref{fig:evta_framework}, we use an instruction prompt to specify the role and prediction objective of eVTA$_0$:

\begin{tcolorbox}[
    colback=gray!5,
    colframe=gray!50,
    boxrule=0.6pt,
    arc=2pt,
    left=6pt,
    right=6pt,
    top=5pt,
    bottom=5pt
]
\small\textit{
You are a reward model for robotic manipulation. Given the task instruction and the ordered history of observations, predict the probability of successful task completion. Return a scalar reward in $[0,1]$, where 1 indicates a high probability of success and 0 indicates a high probability of failure. The multimodal inputs are provided in the following order: task prompt, history observation images, reward token.
}
\end{tcolorbox}

The prompt explicitly specifies the success-probability prediction task, output range, and input ordering for eVTA$_0$.

\noindent\textbf{Training algorithm.}
Algorithm~\ref{alg:evta-training} summarizes the training algorithm of eVTA$_0$. For terminal transitions, the target is given directly by the observed task outcome; for non-terminal transitions, the target is the next-step success prediction from a slowly updated target eVTA$_0$. In this way, sparse terminal supervision is progressively propagated to intermediate timesteps, implementing the recursive regression characterized in Theorem~\ref{theorem:bootstrap-consistency}.

\noindent \textbf{Dataset augmentation.} 
While eVTA$_0$ can be trained using binary terminal outcomes alone, we adopt an additional trajectory-level augmentation strategy  for failed rollouts. This is because  assigning a hard target of zero  for failed trajectories can be overly conservative, as some failures may occur after the maximum episode length is reached despite exhibiting behaviors close to task completion. Therefore, following the strategy proposed in \citep{fei2026srpo}, we augment the training dataset by assigning soft terminal targets to failed trajectories. 
The key idea is to estimate the success likelihood of failed trajectories rather than treating all failures as equally unsuccessful. Specifically, we use a pretrained world model to extract trajectory-level representations and apply DBSCAN to cluster successful trajectories into representative success centers. For each failed trajectory, its soft terminal target $z\in[0,1]$ is computed based on the Euclidean distance to the nearest success center, normalized by the mean and standard deviation of distances across failed trajectories, and mapped through a sigmoid function. This augmentation assigns higher terminal targets to failed trajectories that exhibit behaviors closer to successful executions, providing more informative supervision for training eVTA$_0$.

\section{Details of Experimental Setup}
\label{appendix:experimental-setup}
\subsection{Reward Quality Evaluation}
\label{appendix:reward-quality-evaluation}
\noindent\textbf{Evaluation metrics.}
We provide additional details on the four metrics used in the reward quality evaluation. 
For temporal consistency, we use VOC~\citep{ma2024vision} and VROC~\citep{zhai2025vision}. 
VOC measures the consistency between reward predictions and the forward progression of a demonstration, while VROC examines whether the predicted reward decreases when the demonstration is temporally reversed. Together, they characterize whether a reward model responds consistently as the trajectory moves toward or away from task completion. Both metrics are computed on expert demonstrations, with higher values indicating stronger temporal consistency. 
For outcome distinction, we use MSE and Kendall's $\tau_a$~\citep{liang2026robometer}. 
MSE measures the discrepancy between predicted terminal rewards and the associated terminal success probabilities, while Kendall's $\tau_a$ measures how well the predicted rewards distinguish successful trajectories from failed ones. Together, these metrics assess whether a reward model can reliably distinguish successful and failed trajectories and accurately reflect terminal success likelihood. 
Both metrics are computed on mixed-quality rollouts. Lower MSE and higher Kendall's $\tau_a$ indicate better outcome distinction. When computing MSE, we use the final reward prediction of each trajectory and normalize model outputs to $[0,1]$ when necessary, following the corresponding inference method of each model. 

\noindent\textbf{Evaluation datasets.} We evaluate reward quality on two robotic manipulation benchmarks, LIBERO and MetaWorld. LIBERO consists of four task suites, LIBERO-Spatial, LIBERO-Object, LIBERO-Goal, and LIBERO-Long, each containing 10 tasks. For MetaWorld, we evaluate 30 manipulation tasks and follow prior work~\citep{chen2025pirl} to organize them into four difficulty levels: Easy, Medium, Hard, and Very-Hard. Each of the Easy and Medium suites contains 10 tasks, while each of the Hard and Very-Hard suites contains 5 tasks.

For temporal-consistency evaluation, we use expert demonstrations provided by both benchmarks. On LIBERO, we randomly sample 10 demonstrations per task, resulting in 400 demonstrations across 40 tasks. On MetaWorld, we similarly sample 400 demonstrations across 30 tasks: 10 demonstrations per task for the Easy and Medium suites, and 20 demonstrations per task for the Hard and Very-Hard suites. For outcome-distinction evaluation, we use the $\pi_0$~\citep{black2024pi_0} policy fine-tuned with RLinf~\citep{yu2026rlinf}\footnote{we use the publicly released RLinf $\pi_0$ SFT checkpoints for LIBERO and MetaWorld: \url{https://huggingface.co/RLinf/RLinf-Pi0-LIBERO-Spatial-Object-Goal-SFT} and \url{https://huggingface.co/RLinf/RLinf-Pi0-MetaWorld-SFT}, respectively.} to collect online rollouts on both LIBERO and MetaWorld. We collect 400 rollouts for each benchmark as the evaluation datasets. The LIBERO set contains 201 successful and 199 failed trajectories, while the MetaWorld set contains 230 successful and 170 failed trajectories.

\noindent\textbf{Fine-tuning VLAC and Robometer.}
Although VLAC~\citep{zhai2025vision} and Robometer~\citep{liang2026robometer} are pretrained on large-scale demonstration datasets, we empirically find that their performance degrades substantially when directly applied to LIBERO and MetaWorld without downstream adaptation. We therefore follow their original training recipes and fine-tune both models on benchmark-specific data. Specifically, we randomly sample additional 400 demonstrations from each benchmark, without  overlap with the evaluation sets. In particular, Robometer primarily learns progress from expert demonstrations and additionally incorporates failed trajectories through an auxiliary trajectory-comparison objective. Thus, we follow the original paper~\citep{liang2026robometer} and generate failure data by injecting random perturbations into demonstration actions and executing the resulting perturbed trajectories in the environment. This procedure yields 452 additional failed trajectories for each benchmark. We fine-tune both models for 5 epochs using LoRA~\citep{hu2021lora} with rank 32 and set the batch size and learning rate to $32$ and $2\times10^{-5}$, respectively.

\begin{table}[t]
    \centering
    \footnotesize
    \setlength{\tabcolsep}{5pt}
    \renewcommand{\arraystretch}{1.02}
    \caption{RL and $\pi_{0.5}$ training configurations for the policy learning experiments on  LIBERO-Spatial / LIBERO-Object / LIBERO-Goal / LIBERO-Long.}
    \label{tab:policy-training-config}
    \begin{tabular}{@{}p{0.43\textwidth}p{0.51\textwidth}@{}}
        \toprule
        \textbf{Parameter} & \textbf{Configuration} \\
        \midrule
        \multicolumn{2}{@{}l}{\textit{RL configuration}} \\
        RL algorithm & GRPO\\
        Advantage estimator & Group-relative advantages\\
        Training budget & 50 epochs with 6,400 interaction episodes\\
        Parallel environments & 16\\
        Rollouts per epoch  & 8 \\
        Episodes per epoch & 128 \\
        GRPO group size & 8 \\
        Maximum episode length & 240 / 240 / 320 / 480 steps \\
        Batch size & 4 \\
        Gradient accumulation steps & 16 \\
        Policy learning rate & $2\!\times\!10^{-6}$ \\
        Adam $(\beta_1,\beta_2,\epsilon)$ & $(0.9,\,0.95,\,10^{-8})$ \\
        Weight decay  & 0.01  \\
        Gradient clipping & 1.0 \\
        Discount factor $\gamma$ & 0.99 \\
        Policy clipping (lower / upper) & 0.2 / 0.2 \\
        Update epochs & 2 \\
        Advantage normalization & Enabled \\
        Train action expert only & Enabled \\
        Random seed & 42 \\

        \addlinespace[2pt]
        \multicolumn{2}{@{}l}{\textit{$\pi_{0.5}$ configuration}} \\
        Initial policy & RLinf-$\pi_{0.5}$-LIBERO-SFT \\
        Action chunk length & 10 \\
        Denoising steps & 5 \\
        \bottomrule
    \end{tabular}
\end{table}

\noindent\textbf{Implementation of eVTA$_0$.}
We use Qwen3-VL-4B-Instruct~\citep{bai2025qwen3} as the backbone of eVTA$_0$ throughout all experiments. 
We set the observation context window to 5 with a temporal stride of 2. 
Notably, the visual observations consist of both a third-person camera image and a wrist-mounted camera image, providing complementary global and egocentric views of the manipulation process. For each benchmark, we use the SFT $\pi_0$ policy released by RLinf to collect 1,600 policy rollouts for training eVTA$_0$. The LIBERO dataset contains 890 successful and 710 failed trajectories, while the MetaWorld dataset follows the same composition. The rollouts cover all tasks in each benchmark and are collected independently from the evaluation sets, with no trajectory overlap.  We fine-tune eVTA$_0$ for 5 epochs using LoRA with rank 16 and configure the batch size and learning rate as $32$ and $1\times10^{-5}$, respectively. The target network is updated every $K=16$ training steps with a soft-update rate of $\alpha=0.1$. Note that unlike VLAC and Robometer, eVTA$_0$ does not rely on pretraining over large-scale robotic demonstration datasets for reward modeling. Instead, it learns effective reward signals directly from mixed-quality policy rollouts. 

\subsection{Policy Learning Evaluation}
\label{appendix:policy-learning-evaluation}
Following previous work~\citep{chen2025pirl}, we use $\pi_{0.5}$ as the base policy for experiments and initialize it from the few-shot SFT checkpoint released by RLinf. This checkpoint is obtained by fine-tuning $\pi_{0.5}$ with 40 LIBERO demonstrations, corresponding to one demonstration per task across the four task suites. Such a few-shot initialization ensures a positive success rate for each task while preserving substantial room for subsequent RL-based policy improvement. Since our goal is to evaluate how different rewards affect subsequent RL optimization, the few-shot checkpoint provides a more informative setting for distinguishing the effects of different reward signals. Therefore, all reward methods start from the same SFT checkpoint and are trained under the same RL configurations, ensuring that performance differences primarily reflect the effect of the reward signal. Unless otherwise specified, the four LIBERO suites use the same training and evaluation protocol. The detailed configurations are summarized in Tables~\ref{tab:policy-training-config} and~\ref{tab:policy-eval-config}.

\begin{table}[t]
    \centering
    \footnotesize
    \setlength{\tabcolsep}{5pt}
    \renewcommand{\arraystretch}{1.02}
    \caption{$\pi_{0.5}$ evaluation configurations for the policy learning experiments on  LIBERO-Spatial / LIBERO-Object / LIBERO-Goal / LIBERO-Long.}
    \label{tab:policy-eval-config}
    \begin{tabular}{@{}p{0.43\textwidth}p{0.51\textwidth}@{}}
        \toprule
        \textbf{Parameter} & \textbf{Configuration} \\
        \midrule
        Number of tasks  & 10 / 10 / 10 / 10 \\
        Evaluation rollouts per task & 50 \\
        Total rollouts  & 500 / 500 / 500 / 500 \\
        Maximum episode length & 240 / 240 / 320 / 480 steps \\
        Action chunk length & 10 \\
        Denoising steps & 5 \\
        Random seed & 42 \\
        \bottomrule
    \end{tabular}
\end{table}

\begin{table}[t]
    \centering
    \footnotesize
    \setlength{\tabcolsep}{5pt}
    \renewcommand{\arraystretch}{0.98}
    \caption{Configurations of  SFT policies in real-world experiments (\textit{Pick Carrot} / \textit{Put Shuttlecock}).}
    \label{tab:real-world-sft-config}
    \begin{tabular}{@{}p{0.43\textwidth}p{0.51\textwidth}@{}}
        \toprule
        \textbf{Parameter} & \textbf{Configuration} \\
        \midrule
        Training epochs & 5 \\
        Expert demonstrations & 20 / 60 \\
        Batch size & 4 \\
        Gradient accumulation steps & 2 \\
        Learning rate & $1\!\times\!10^{-5}$ \\
        Gradient clipping & 1.0 \\
        Adam $(\beta_1,\beta_2,\epsilon)$ & $(0.9,\,0.95,\,10^{-8})$ \\
        Fine-tuning strategy & Full-model fine-tuning \\

        \addlinespace[2pt]
        \multicolumn{2}{@{}l}{\textit{$\pi_{0.5}$ training/evaluation configuration}} \\
        Action chunk length & 32 \\
        Denoising steps & 5 \\
        \bottomrule
    \end{tabular}
\end{table}

\subsection{Real-World Evaluation}
\label{appendix:real-world-evaluation}
\noindent\textbf{Physical setup.}
Figure~\ref{fig:real_world_configurations} shows the physical platform of our real-world experiments. The platform consists of a Franka Research 3 robotic arm equipped with a parallel gripper and two Intel RealSense D435i RGB-D cameras. One external camera provides a third-person view of the overall workspace, while the other wrist-mounted camera captures local  information around the manipulation area.

\noindent\textbf{Task design and configurations.}
We design two representative manipulation tasks covering different  scenarios.
The first task requires the robot to pick up a carrot from a moving conveyor and place it on the plate. By introducing continuously changing object positions, this task evaluates the policy's adaptation ability in dynamic environments.
The second task requires the robot to put a shuttlecock from the table onto the one  in the plate. Since the shuttlecock is lightweight and difficult to grasp stably, small grasping deviations can cause the object to slip or shift, making this task suitable for evaluating grasp stability and precise control.
As depicted in Figure~\ref{fig:real_world_id_ood_configurations}, to further evaluate generalization, we construct several OOD settings by varying object positions, distances, rotations, and appearances.

\noindent\textbf{Evaluation method.} In our real-world experiments, each task is evaluated with 50 trials covering both in-distribution (ID) and out-of-distribution (OOD) scenarios.
For the \textit{Pick Carrot} task, the evaluation consists of 25 ID trials and 25 OOD trials. Specifically, we perform 5 trials under the OOD-Distance setting and 4 trials for each of the five OOD-Rotation settings. For the \textit{Put Shuttlecock} task, we conduct 30 ID trials and 20 OOD trials, including 10 trials under OOD-Color and 10 trials under OOD-Position.

\begin{figure}[t]
    \centering
    \includegraphics[width=0.95\textwidth]{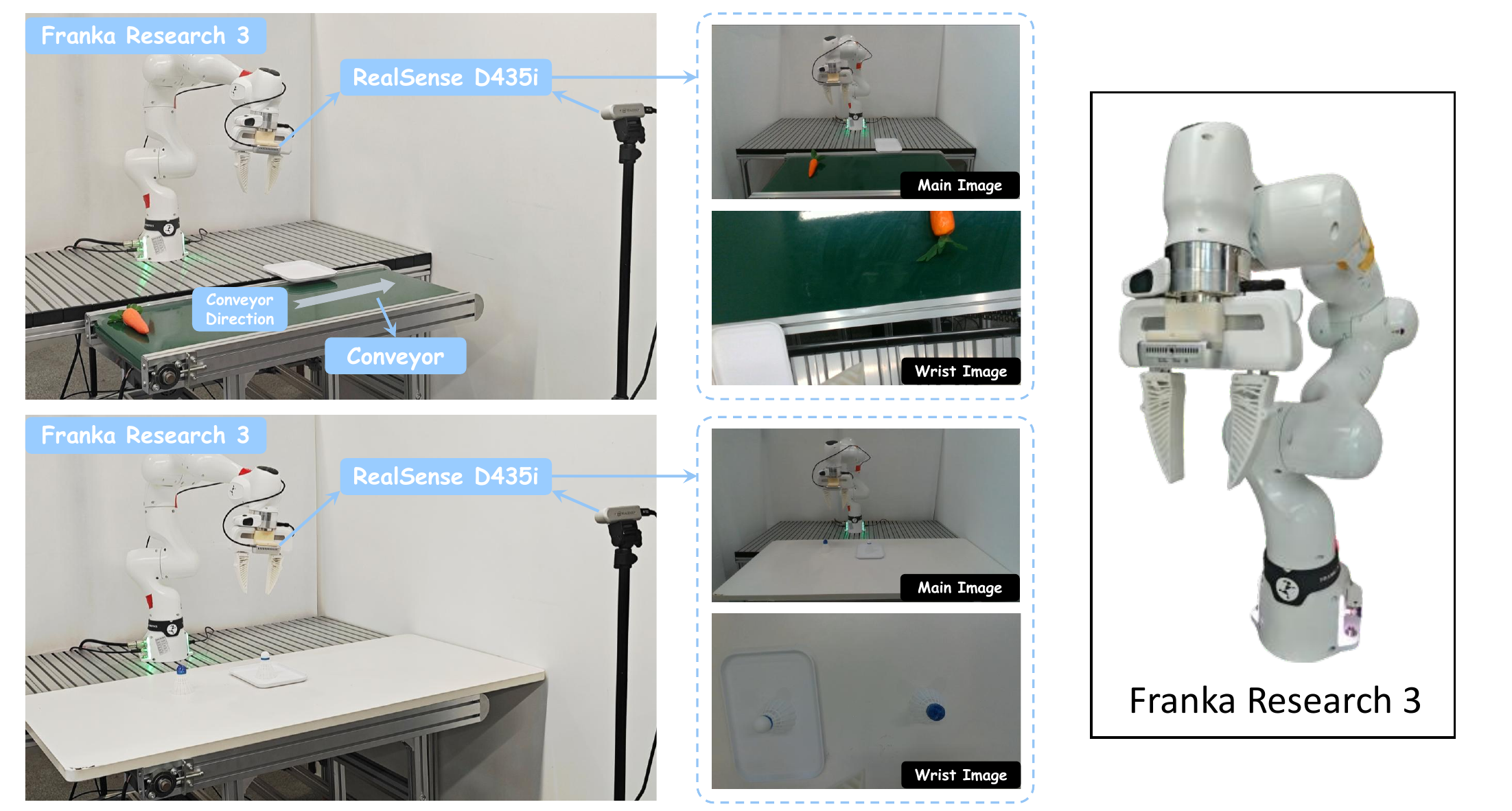}
    \vspace{-0.2cm}
    \caption{The physical workspace used in our real-world experiments, which is equipped with a Franka Research 3 robot arm, two Intel RealSense D435i cameras and optionally a moving conveyor.}
    \label{fig:real_world_configurations}
    \vspace{-0.2cm}
 \end{figure}

 \begin{figure}[t]
    \centering
    \includegraphics[width=0.95\textwidth]{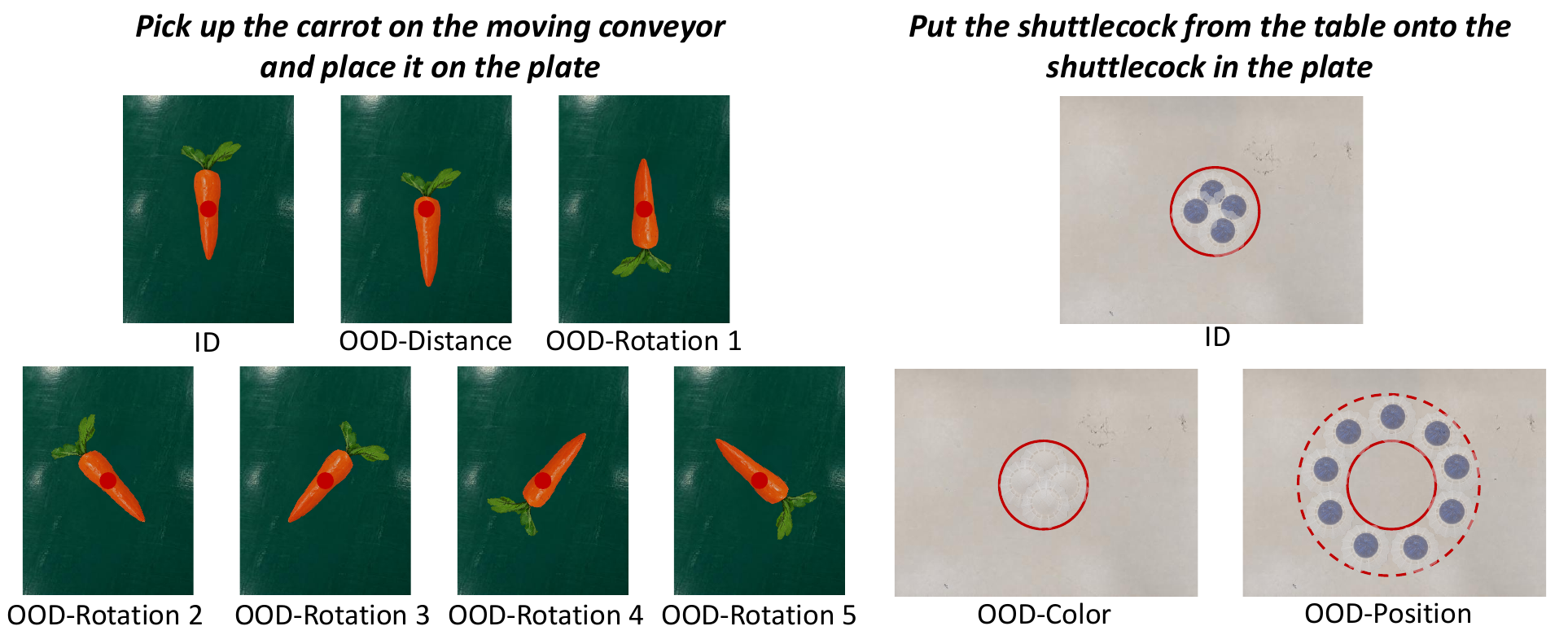}
    \vspace{-0.2cm}
    \caption{The ID and OOD configurations for the real-world experiments. }
    \label{fig:real_world_id_ood_configurations}
    \vspace{-0.2cm}
 \end{figure}

\begin{table}[t]
    \centering
    \footnotesize
    \setlength{\tabcolsep}{5pt}
    \renewcommand{\arraystretch}{0.98}
    \caption{Training configurations for IQL and eVTA$_0$ in real-world  experiments (\textit{RLER Round 1} / \textit{RLER Round 2} / \textit{Fixed eVTA$_0$)}.}
    \label{tab:real-world-iql-srm-config}
    \label{tab:real-world-configurations}
    \begin{tabular}{@{}p{0.43\textwidth}p{0.51\textwidth}@{}}
        \toprule
        \textbf{Parameter} & \textbf{Configuration} \\
        \midrule
        \multicolumn{2}{@{}l}{\textit{IQL configuration}} \\
        Number of trajectories  & 50 / 50 / 100 \\
        Training epochs & 30  \\
        Batch size & 4 \\
        Gradient accumulation steps & 4 \\
        Discount factor $\gamma$ & 1.0 \\
        Expectile $\tau$ & 0.7 \\
        Advantage temperature $\beta$ & 0.3 \\
        Maximum advantage weight & 20.0 \\
        Target soft-update rate & 0.005 \\
        Policy learning rate & $1\!\times\!10^{-6}$ \\
        Value learning rate & $1\!\times\!10^{-4}$ \\
        Critic learning rate & $1\!\times\!10^{-4}$ \\
        Adam $(\beta_1,\beta_2,\epsilon)$ & $(0.9,\,0.95,\,10^{-8})$ \\
        Policy fine-tuning strategy & Full-model fine-tuning \\

        \addlinespace[2pt]
        \multicolumn{2}{@{}l}{\textit{eVTA$_0$ configuration}} \\
        Backbone & Qwen3-VL-4B-Instruct \\
        Training epochs & 30  \\
        Context window & 5  \\
        Fine-tuning strategy & LoRA \\
        LoRA rank & 16 \\
        Learning rate & $1\!\times\!10^{-5}$ \\
        Weight decay & 0.01 \\
        Gradient clipping & 0.5 \\
        Batch size & 32 \\
        Gradient accumulation steps & 2 \\
        Target-update interval & 16 \\
        Target soft-update rate & 0.1 \\
        \bottomrule
    \end{tabular}
\end{table}

\noindent\textbf{Policy training.}
We use $\pi_{0.5}$ as the base policy. For the \textit{Pick Carrot} task, we collect 20 expert demonstrations for SFT, while for the \textit{Put Shuttlecock} task, we collect 60 expert demonstrations. All RL experiments start from the corresponding SFT policy, ensuring that different reward learning methods share the same initial policy.
For eVTA$_0$-RLER, we adopt a closed-loop reward evolution procedure. In the first round, we train eVTA$_0$ using 50 rollouts collected from the SFT policy, and then optimize the policy with IQL based on the trained eVTA$_0$. In the second round, we collect another 50 rollouts from the Round 1 policy, use them to further update eVTA$_0$, and perform another round of policy optimization.
For the fixed eVTA$_0$ baseline, we aggregate the 100 rollouts collected from the SFT policy and the Round 1 policy, train eVTA$_0$ once on this fixed dataset, and then optimize the policy with IQL. 
Notably, the fixed eVTA$_0$ baseline and eVTA$_0$-RLER use the same total amount of rollout data and matched optimization budgets for both reward model training and policy optimization.
Therefore, the key difference between the two settings lies in the training schedule: RLER periodically adapts eVTA$_0$ with newly collected rollouts as the policy evolves, whereas the fixed baseline trains eVTA$_0$ on the aggregated rollout dataset without intermediate reward model adaptation.
This comparison directly evaluates the effectiveness of closed-loop reward evolution in eVTA$_0$-RLER. The detailed configurations are summarized in Table~\ref{tab:real-world-configurations}.

\section{Supplementary Results}
\label{appendix:additional-results}

\subsection{Ablation Study of eVTA$_0$}
\label{appendix:ablation-study}

We further study two key hyperparameters of eVTA$_0$ on LIBERO-Spatial: the context window size $w$ and the target-network update interval $K$. Figure~\ref{fig:evta-temporal-ablation} summarizes their effects on reward quality.
As shown in Figure~\ref{fig:evta-temporal-ablation} (top), increasing $w$ from 1 to 5 substantially improves VOC, VROC, and Kendall's $\tau_a$, with $w=5$ achieving the best overall trade-off across all metrics. However, further enlarging the context window leads to a clear degradation in VROC and Kendall's $\tau_a$, together with an overall increase in MSE, even though VOC generally continues to improve. This suggests that longer histories facilitate the recognition of overall task progression, but excessively long contexts may   retain earlier positive progress even after the current behavior starts to deviate. 
Figure~\ref{fig:evta-temporal-ablation} (bottom) further reports the effect of the target-network update interval $K$. Reward quality shows an overall degradation as $K$ increases. In particular, $K=16$ achieves the best overall performance, whereas larger update intervals generally lead to lower VOC and VROC and higher MSE. Since eVTA$_0$ relies on bootstrapped targets, infrequent target-network updates can make the supervision increasingly stale and slow the propagation of updated success estimates.
Based on these results, we set $w=5$ and $K=16$ as the default configuration of eVTA$_0$ and use them consistently throughout all experiments, as they provide the best overall reward quality.

\begin{figure}[t]
    \centering
    \includegraphics[width=\textwidth]{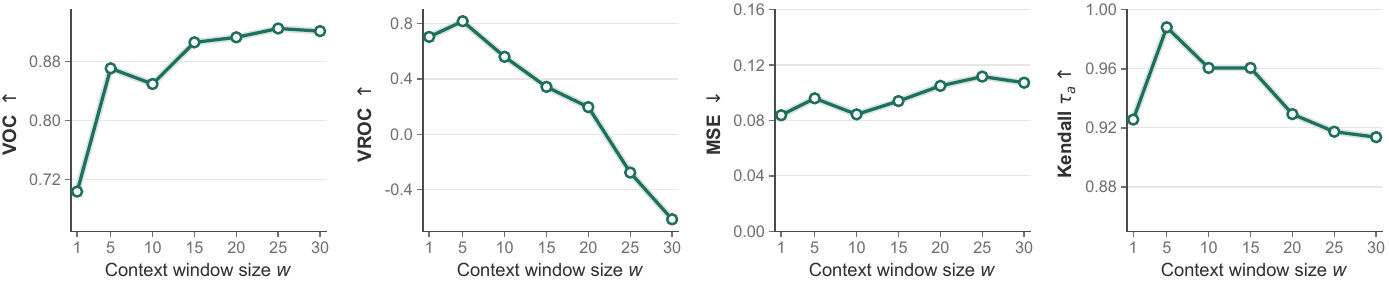}
    \vspace{2mm}
    \includegraphics[width=\textwidth]{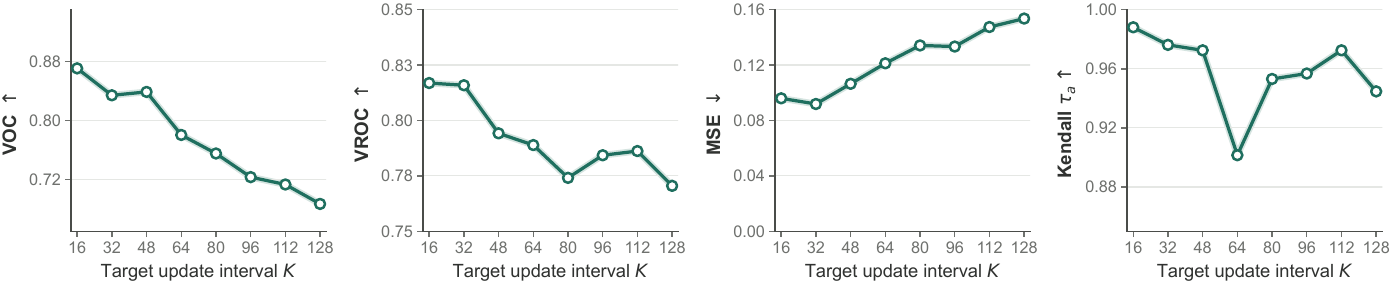}
    \vspace{-0.8cm}
    \caption{Ablation of hyperparameters in eVTA$_0$. The top panel reports the  effect of the context window size $w$, while the bottom part shows the effect of the target-network update interval $K$.}
    \label{fig:evta-temporal-ablation}
\end{figure}

\subsection{Qualitative Analysis of Temporal Reward Dynamics}
\label{appendix:reward-quality-qualitative}

To better understand how eVTA$_0$ provides reward feedback within a trajectory, we visualize its predictions on more successful and failed rollouts in Figures~\ref{fig:additional-success-reward-trajectories} and~\ref{fig:additional-failure-reward-trajectories}. Across both successful and failed trajectories, eVTA$_0$ exhibits non-uniform and behavior-responsive reward dynamics. Its prediction can remain relatively stable during phases with limited impact on task success, change sharply around transitions that substantially improve the likelihood of success, and decrease when the behavior deviates from successful completion.

A representative example is the third successful trajectory in Figure~\ref{fig:additional-success-reward-trajectories}, where the robot grasps and lifts the black bowl by $t_2$, transports it toward the plate from $t_3$ to $t_4$, and completes the placement by $t_5$. The prediction of eVTA$_0$ remains relatively low after the grasp and rises sharply only when later states substantially increase the likelihood of successful completion. This is because a successful grasp alone does not guarantee eventual task success and failures may still occur during subsequent transport or placement (see Figure~\ref{fig:additional-failure-reward-trajectories}). In contrast, Robometer's predictions increase more steadily with task progress. This reflects a fundamental difference in what the two approaches model: progress-based reward models estimate how far the task has advanced, whereas eVTA$_0$ estimates the likelihood of eventual success conditioned on the current history. Consequently, eVTA$_0$ exhibits more non-uniform temporal reward dynamics, with larger changes concentrated around transitions that more strongly affect eventual success.

The third failed trajectory in Figure~\ref{fig:additional-failure-reward-trajectories} provides a complementary example. The robot brings the bowl close to successful placement at $t_2$, subsequently moves away from the target region at $t_3$, partially recovers at $t_4$, and ultimately fails to place the bowl inside the drawer at $t_5$. The prediction of eVTA$_0$ closely reflects this non-monotonic evolution: it increases as the robot approaches success, decreases after the deviation, rises again during recovery, and finally drops as the likelihood of successful completion decreases. Such adaptive adjustment is particularly important for mixed-quality RL rollouts, where intermediate progress does not necessarily evolve monotonically toward success, as the policy may repeatedly approach and move away from successful states while attempting to complete the task. In contrast, Robometer maintains high predictions even after the trajectory deviates from successful completion, revealing a mismatch between its reward estimates and the actual evolution of the task.

Overall, these results show that eVTA$_0$ produces non-uniform and behavior-responsive reward dynamics that reflect changes in the likelihood of task success rather than simply tracking task progress. This property is particularly useful for RL, where policy rollouts may repeatedly approach and move away from successful completion. By assigning larger reward changes to transitions that substantially affect eventual success, eVTA$_0$ provides more informative temporal credit assignment for distinguishing beneficial behaviors, harmful deviations, and subsequent recoveries. 
Interestingly, this behavior-responsive pattern also qualitatively resembles biological reward learning, where dopamine-related teaching signals can increase or decrease as new behaviors change the expected future outcomes~\citep{kasdin2025natural}.
This provides an intuitive perspective on why non-uniform reward feedback can facilitate policy learning.

\begin{figure}[t]
    \centering
    \includegraphics[width=0.99\textwidth]{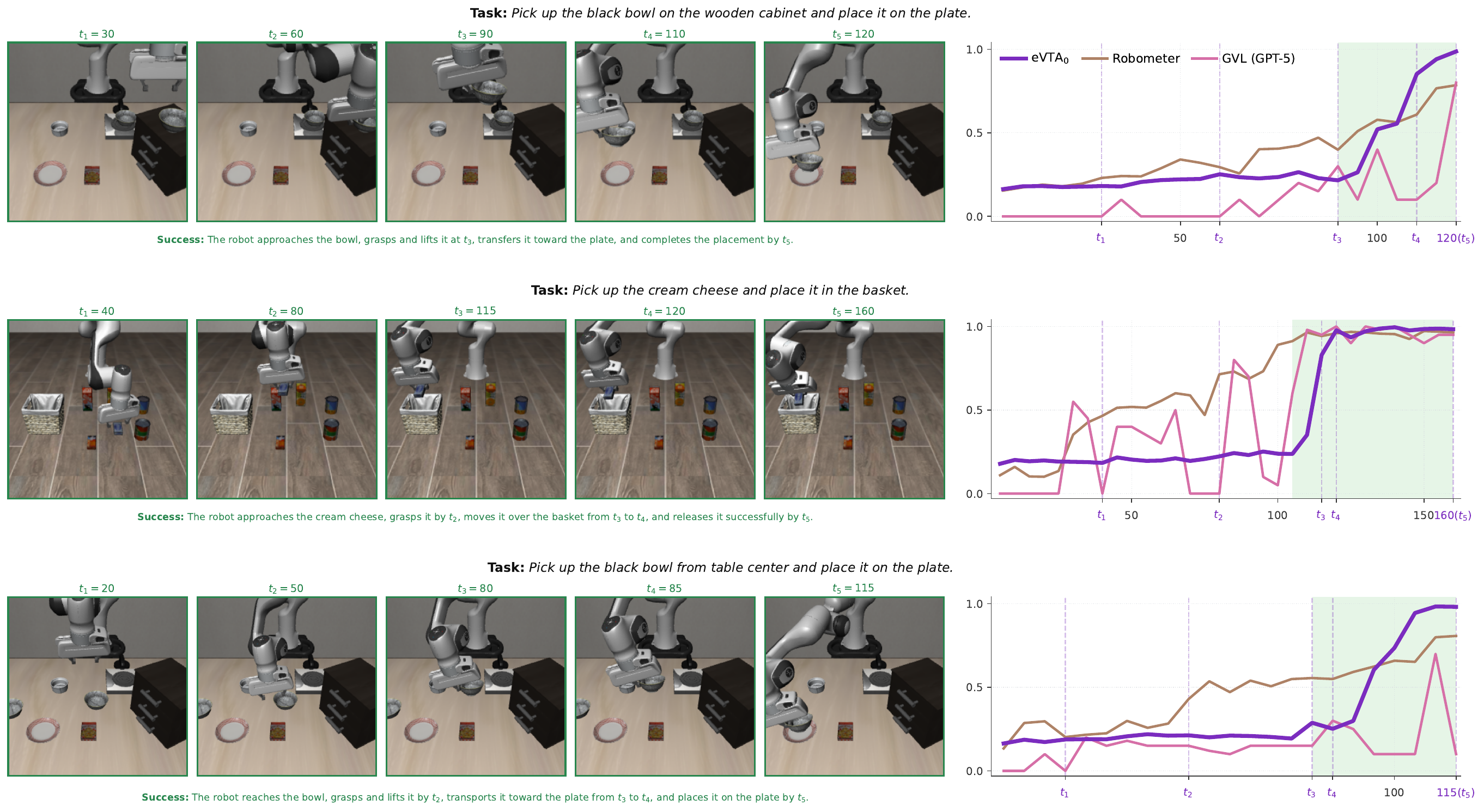}
    \vspace{-0.3cm}
    \caption{Additional qualitative examples of reward predictions on successful trajectories.}
    \label{fig:additional-success-reward-trajectories}
    \vspace{-0.2cm}
\end{figure}

\begin{figure}[t]
    \centering
    \includegraphics[width=0.99\textwidth]{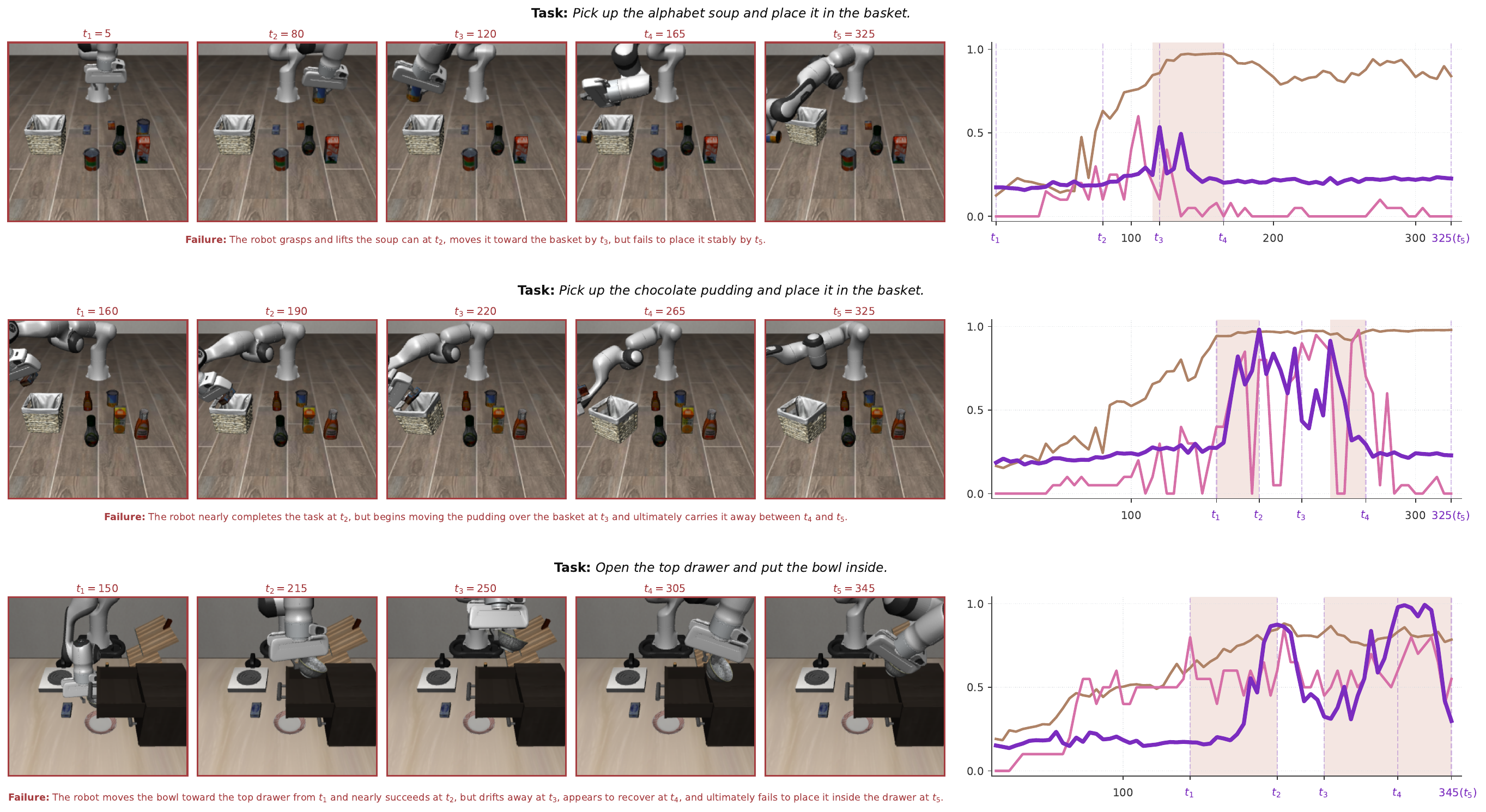}
    \vspace{-0.3cm}
    \caption{Additional qualitative examples of reward predictions on failed trajectories.}
    \label{fig:additional-failure-reward-trajectories}
    \vspace{-0.2cm}
\end{figure}

 \subsection{Detailed Real-World Experiment Results}
 \label{appendix:real-world-experiment-results}

Tables~\ref{tab:real-world-carrot-results} and~\ref{tab:real-world-shuttlecock-results} report the detailed real-world evaluation results. Across both tasks, eVTA$_0$-RLER achieves the best overall performance after Round~2, with more pronounced improvements under OOD conditions. On \textit{Pick Carrot}, the overall success rate increases from 32/50 with SFT to 42/50, while the OOD success rate improves from 8/25 to 17/25. On \textit{Put Shuttlecock}, the corresponding results improve from 30/50 to 43/50 overall and from 9/20 to 16/20 under OOD conditions. eVTA$_0$-RLER Round~2 also consistently outperforms the fixed eVTA$_0$ baseline on both tasks, further demonstrating the benefit of evolving reward feedback during policy optimization. Notably, eVTA$_0$-RLER enables successful task completion in OOD settings where the SFT policy fails across all evaluation trials, improving from 0/4 to 2/4 on OOD-Rotation~3 for \textit{Pick Carrot} and from 0/10 to 7/10 on OOD-Position for \textit{Put Shuttlecock}. These results show that eVTA$_0$-RLER not only improves overall policy performance, but also enables successful task execution under distribution shifts that the initial SFT policy may fail to handle.

Figures~\ref{fig:additional-real-world-carrot} and~\ref{fig:additional-real-world-shuttlecock} provide the real-world rollouts illustrating  the policy behaviors under OOD conditions. For \textit{Pick Carrot}, the OOD-Distance setting places the carrot near the outer range of the robot's workspace. The SFT policy fails to reach sufficiently far to make a stable grasp, whereas the eVTA$_0$-RLER policy reaches out for the target, grasps the carrot by its leaves, and subsequently completes the transfer and placement. For \textit{Put Shuttlecock}, a similar pattern is observed under OOD-Position setting. When the shuttlecock is placed farther from its ID position, the SFT policy fails to reach and grasp the object, while the eVTA$_0$-RLER policy adjusts its reaching motion to the shifted target position, successfully grasps the shuttlecock, and completes the placement.

\begin{table}[t]
    \centering
    \footnotesize
    \setlength{\tabcolsep}{3.5pt}
    \renewcommand{\arraystretch}{1.08}
    \caption{Detailed real-world results for \textit{``pick up the carrot on the
    moving conveyor and place it on the plate''}. Each entry reports the number
    of successful trials over the total number of trials.}
    \label{tab:real-world-carrot-results}
    \resizebox{\textwidth}{!}{%
    \begin{tabular}{@{}l*{8}{c}@{}}
        \toprule
        \textbf{Policy}
        & \textbf{ID}
        & \textbf{\shortstack{OOD\\Distance}}
        & \textbf{\shortstack{OOD\\Rotation 1}}
        & \textbf{\shortstack{OOD\\Rotation 2}}
        & \textbf{\shortstack{OOD\\Rotation 3}}
        & \textbf{\shortstack{OOD\\Rotation 4}}
        & \textbf{\shortstack{OOD\\Rotation 5}}
        & \textbf{Average} \\
        \midrule
        SFT
        & 24/25 & 1/5 & 3/4 & 1/4 & 0/4 & 2/4 & 1/4 & 32/50 \\
        eVTA$_0$-RLER (Round 1)
        & 25/25 & 0/5 & 4/4 & 0/4 & 3/4 & 3/4 & 0/4 & 35/50 \\
        eVTA$_0$-RLER (Round 2)
        & 25/25 & 5/5 & 4/4 & 2/4 & 2/4 & 3/4 & 1/4 & 42/50 \\
        Fixed eVTA$_0$
        & 24/25 & 1/5 & 4/4 & 1/4 & 2/4 & 2/4 & 2/4 & 36/50 \\
        \bottomrule
    \end{tabular}%
    }
\end{table}

\begin{table}[t]
    \centering
    \fontsize{8.2pt}{9.6pt}\selectfont
    \renewcommand{\arraystretch}{1.08}
    \caption{Detailed real-world results for \textit{``put the shuttlecock from the
    table onto the shuttlecock in the plate''}. Each entry reports the number
    of successful trials over the total number of trials.}
    \label{tab:real-world-shuttlecock-results}
    \begin{tabular*}{\textwidth}{@{\extracolsep{\fill}}lcccc@{}}
        \toprule
        \textbf{Policy}
        & \textbf{ID}
        & \textbf{OOD-Color}
        & \textbf{OOD-Position}
        & \textbf{Average} \\
        \midrule
        SFT
        & 21/30 & 9/10 & 0/10 & 30/50 \\
        eVTA$_0$-RLER (Round 1)
        & 26/30 & 9/10 & 3/10 & 38/50 \\
        eVTA$_0$-RLER (Round 2)
        & 27/30 & 9/10 & 7/10 & 43/50 \\
        Fixed eVTA$_0$
        & 26/30 & 9/10 & 4/10 & 39/50 \\
        \bottomrule
    \end{tabular*}
\end{table}


\begin{figure}[t]
    \centering
    \includegraphics[width=\textwidth]{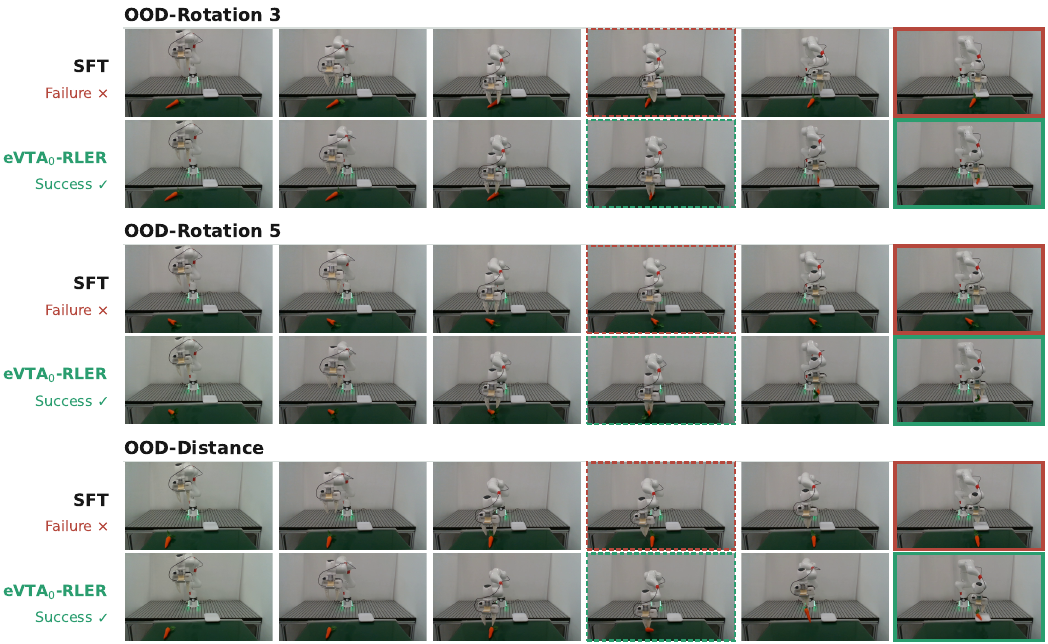}
    \vspace{-0.5cm}
    \caption{Real-world rollouts under OOD conditions for \textit{``pick up the carrot on the
    moving conveyor and place it on the plate''.}}
    \label{fig:additional-real-world-carrot}
\end{figure}

\begin{figure}[t]
    \centering
    \includegraphics[width=\textwidth]{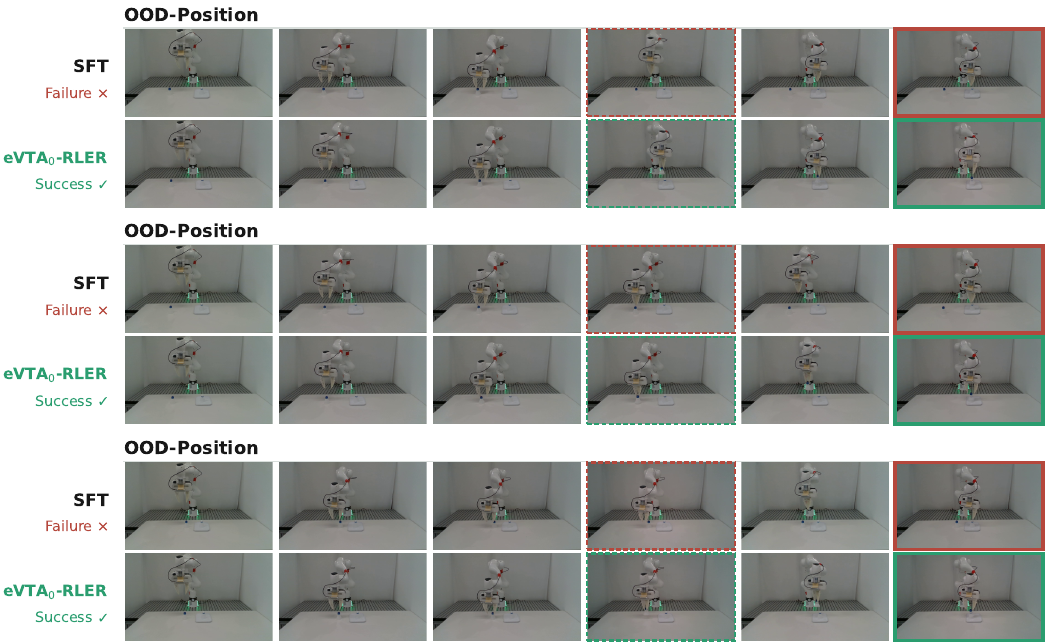}
    \vspace{-0.5cm}
    \caption{Real-world  rollouts under  OOD     conditions for \textit{``put the shuttlecock from the
    table onto the shuttlecock in the plate''}.}
    \label{fig:additional-real-world-shuttlecock}
\end{figure}


\section{Discussion}
\label{appendix:discussion}
\subsection{Relationship between eVTA$_0$ and Value Functions}
\label{appendix:evta-vs-critic}

eVTA$_0$ is mathematically related to a value function under specific settings. Recall from Eq.~\ref{eq:success-probability-reward} that the success-probability reward is defined as
\[
r_t
=
\mathbb{P}_\pi(S=1\mid\mathcal{F}_t)
=
\mathbb{E}_\pi[S\mid\mathcal{F}_t].
\]
When the task reward consists only of a terminal success signal and the discount factor is set to $\gamma=1$, the expected return conditioned on the trajectory history becomes
\[
V^\pi(\mathcal{F}_t)
=
\mathbb{E}_\pi[S\mid\mathcal{F}_t]
=
\mathbb{P}_\pi(S=1\mid\mathcal{F}_t),
\]
which is equivalent to the success-probability reward approximated by eVTA$_0$.

Despite this connection, eVTA$_0$ serves a different role from the critic used for value estimation in standard RL algorithms. The key distinction lies in what they provide to policy optimization. A critic estimates the return induced by a given reward function and is used internally for value or advantage estimation. In contrast, eVTA$_0$ learns dense success-probability feedback from policy rollouts and task outcomes, which is transformed by Eq.~\ref{eq:reward-transformation} into the reward signal used for  RL. The RL algorithm may still maintain its own critic to estimate the return induced by this reward. Therefore, eVTA$_0$ complements rather than replaces the critic. Furthermore, they  differ in how they are coupled with policy optimization. A critic is typically updated as an internal component of the RL algorithm, whereas eVTA$_0$ is maintained as a separate reward model. It can remain fixed throughout policy optimization or be periodically adapted with newly collected rollouts under RLER.

\begin{figure}[t]
    \centering
    \includegraphics[width=0.99\textwidth]{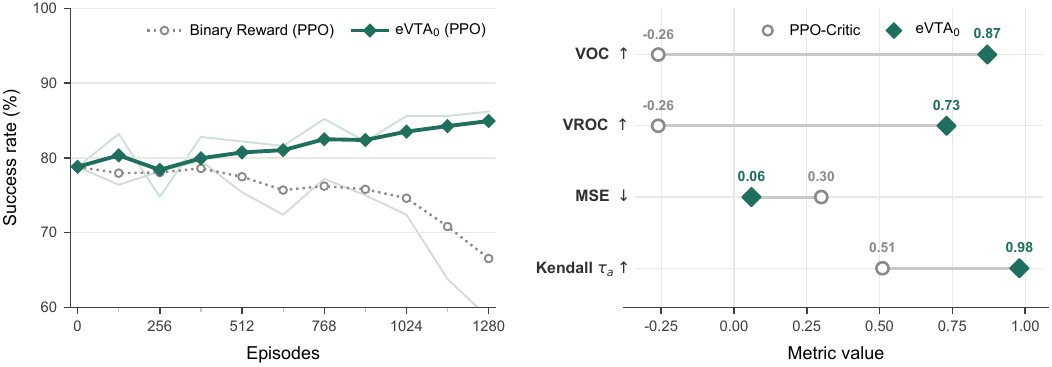}
    \vspace{-0.3cm}
    \caption{
Comparison of implicit value estimation and explicit success-probability reward learning. 
The left panel illustrates policy learning curves of PPO with sparse binary rewards and eVTA$_0$ rewards. The right panel depicts the estimation quality of the learned signals from the PPO critic and eVTA$_0$. 
}
    \label{fig:critic_comparison}
\end{figure}

\noindent\textbf{Empirical comparison with RL critics.}
Given the theoretical connection above, we further compare eVTA$_0$ with the value estimation performed by standard RL critics. We choose PPO with binary terminal rewards and $\gamma=1$ as the baseline, since under this setting the learned value function has the same target quantity as the success-probability formulation of eVTA$_0$. Specifically, we conduct experiments on LIBERO-Spatial and train $\pi_{0.5}$ with PPO for 1,280 interaction episodes. Detailed  training configurations are provided in Table~\ref{tab:critic_comparison}.

\begin{table}[t]
    \centering
    \footnotesize
    \setlength{\tabcolsep}{5pt}
    \renewcommand{\arraystretch}{1.02}
    \caption{PPO training configurations for the critic comparison experiment on LIBERO-Spatial in Appendix~\ref{appendix:evta-vs-critic}. The binary reward and eVTA$_0$ reward use the same settings.}
    \label{tab:critic_comparison}
    \begin{tabular}{@{}p{0.43\textwidth}p{0.51\textwidth}@{}}
        \toprule
        \textbf{Parameter} & \textbf{Configuration} \\
        \midrule
        Training budget & 10 epochs with 1,280 interaction episodes \\
        Parallel environments & 16 \\
        Rollouts per epoch & 8 \\
        Episodes per epoch & 128 \\
        Maximum episode length & 240 steps \\
        Batch size & 4 \\
        Gradient accumulation steps & 16 \\
        Policy learning rate & $5\!\times\!10^{-6}$ \\
        Value learning rate & $1\!\times\!10^{-4}$ \\
        Adam $(\beta_1,\beta_2,\epsilon)$ & $(0.9,\,0.95,\,10^{-8})$ \\
        Weight decay & 0.01 \\
        Gradient clipping & 1.0 \\
        Discount factor $\gamma$ & 1.0 \\
        GAE $\lambda$ & 0.95 \\
        PPO policy clipping & 0.2 \\
        PPO value clipping & 0.2 \\
        Update epochs & 1 \\
        Train action expert only & Enabled \\
        Random seed & 42 \\
        \bottomrule
    \end{tabular}
\end{table}

As shown in Figure~\ref{fig:critic_comparison} (left), PPO with binary rewards struggles to improve and eventually degrades under sparse binary feedback, whereas replacing the sparse reward with eVTA$_0$ enables stable and continuous policy improvement throughout training. This suggests that explicitly learning success-probability feedback provides more effective credit assignment than relying solely on the implicit value estimation of the PPO critic.
We further evaluate the estimation quality of the learned signals from the PPO critic and eVTA$_0$. As shown in Figure~\ref{fig:critic_comparison} (right), eVTA$_0$ consistently achieves better performance in VOC, VROC, MSE, and Kendall $\tau_a$. These results demonstrate that although the PPO critic and eVTA$_0$ are theoretically related under sparse terminal rewards, explicitly learning a standalone success-probability reward model provides a more accurate and reliable estimation of task success, thereby offering a stronger learning signal for downstream policy learning.


\begin{table}[t]
    \centering
    \small
    \renewcommand{\arraystretch}{1.08}
    \caption{Reward-quality comparison on LIBERO  between temporal-difference (TD) bootstrapping and Monte Carlo (MC) regression. The best result is shown in bold.}
    \label{tab:td-vs-mc}
    \begin{tabular*}{\textwidth}{@{\extracolsep{\fill}}llcccc@{}}
        \toprule
        \textbf{Task suite} & \textbf{Method} & \textbf{VOC} $\uparrow$ & \textbf{VROC} $\uparrow$ & \textbf{MSE} $\downarrow$ & \textbf{Kendall's $\tau_a$} $\uparrow$ \\
        \midrule
        \multirow{2}{*}{LIBERO-Spatial}
            & TD & \textbf{0.87} & \textbf{0.73} & 0.06 & \textbf{0.98} \\
            & MC & 0.24 & 0.17 & \textbf{0.05} & 0.96 \\
        \addlinespace[2pt]
        \multirow{2}{*}{LIBERO-Object}
            & TD & \textbf{0.90} & \textbf{0.88} & 0.03 & \textbf{1.00} \\
            & MC & 0.35 & 0.51 & \textbf{0.03} & 0.98 \\
        \addlinespace[2pt]
        \multirow{2}{*}{LIBERO-Goal}
            & TD & \textbf{0.84} & \textbf{0.74} & 0.07 & \textbf{1.00} \\
            & MC & 0.35 & 0.32 & \textbf{0.04} & 0.97 \\
        \addlinespace[2pt]
        \multirow{2}{*}{LIBERO-Long}
            & TD & \textbf{0.77} & \textbf{0.73} & 0.08 & \textbf{0.93} \\
            & MC & 0.53 & 0.48 & \textbf{0.05} & 0.80 \\
        \midrule
        \multirow{2}{*}{\textbf{Average}}
            & TD & \textbf{0.84} & \textbf{0.77} & 0.06 & \textbf{0.98} \\
            & MC & 0.37 & 0.37 & \textbf{0.04} & 0.93 \\
        \bottomrule
    \end{tabular*}
\end{table}

\subsection{TD Bootstrapping versus Monte Carlo Regression}
\label{app:td-vs-mc}
\noindent\textbf{Theoretical analysis.}
One natural alternative to the TD-style bootstrapping adopted in eVTA$_0$ is Monte Carlo (MC) regression. For each intermediate timestep $t<T$, recall that the success-probability reward is defined as:
\begin{equation*}
    r_t
    =
    \mathbb{E}_{\pi}[S\mid\mathcal{F}_t]
    =
    \mathbb{P}_{\pi}(S=1\mid\mathcal{F}_t),
\end{equation*}
where $S \in \{0,1\}$ denotes the terminal task outcome. Under MC regression, the same terminal outcome is directly used as the supervision target for every intermediate trajectory prefix:
\begin{equation*}
y_t^{\mathrm{MC}} := S.
\end{equation*}
This is a valid estimator of the success-probability target, since:
\begin{equation*}
    \mathbb{E}_{\pi}
    \left[
    y_t^{\mathrm{MC}}
    \mid\mathcal{F}_t
    \right]
    =
    \mathbb{E}_{\pi}
    \left[
    S\mid\mathcal{F}_t
    \right]
    =
    r_t.
\end{equation*}
Recall  the ideal one-step TD target:
\begin{equation*}
    y_t^{\mathrm{TD}}
:=
r_{t+1}
=
\mathbb{E}_{\pi}
\left[
S\mid\mathcal{F}_{t+1}
\right].
\end{equation*}
By the temporal consistency property in Theorem~1, the ideal TD target also has the same conditional expectation as the MC target:
\begin{equation*}
    \mathbb{E}_{\pi}
    \left[
    y_t^{\mathrm{TD}}
    \mid\mathcal{F}_t
    \right]
    =
    \mathbb{E}_{\pi}
    \left[
    r_{t+1}\mid\mathcal{F}_t
    \right]
    =
    r_t.
\end{equation*}
Therefore, MC regression and  TD bootstrapping share the same population target. The main difference between the two approaches  lies in the variance of the supervision signal.

Under MC regression, each trajectory prefix is supervised by a single binary realization of the underlying success probability. Specifically, all prefixes from a successful trajectory receive a target of one, whereas all prefixes from a failed trajectory receive a target of zero. Since $S$ is binary, the conditional variance of the MC target is:
\begin{equation*}
\mathrm{Var}_{\pi}
\left(
y_t^{\mathrm{MC}}
\mid\mathcal{F}_t
\right)
=
\mathrm{Var}_{\pi}
\left(
S\mid\mathcal{F}_t
\right)
=
r_t(1-r_t).
\end{equation*}
Moreover, applying the conditional law of total variance to $S$ gives:
\begin{equation*}
    \mathrm{Var}_{\pi}
    \left(
    S\mid\mathcal{F}_t
    \right)
    =
    \mathbb{E}_{\pi}
    \left[
    \mathrm{Var}_{\pi}
    \left(
    S\mid\mathcal{F}_{t+1}
    \right)
    \mid\mathcal{F}_t
    \right]
    +
    \mathrm{Var}_{\pi}
    \left(
    \mathbb{E}_{\pi}
    \left[
    S\mid\mathcal{F}_{t+1}
    \right]
    \mid\mathcal{F}_t
    \right).
\end{equation*}
Using $r_{t+1}=\mathbb{E}_{\pi}[S\mid\mathcal{F}_{t+1}]$, this becomes:
\begin{equation*}
    \mathrm{Var}_{\pi}
    \left(
    S\mid\mathcal{F}_t
    \right)
    =
    \mathbb{E}_{\pi}
    \left[
    \mathrm{Var}_{\pi}
    \left(
    S\mid\mathcal{F}_{t+1}
    \right)
    \mid\mathcal{F}_t
    \right]
    +
    \mathrm{Var}_{\pi}
    \left(
    r_{t+1}\mid\mathcal{F}_t
    \right).    
\end{equation*}
Since $\mathbb{E}_{\pi}\left[\mathrm{Var}_{\pi}\left(S\mid\mathcal{F}_{t+1}\right)\mid\mathcal{F}_t\right]$ is non-negative, we have:
\begin{equation*}
    \mathrm{Var}_{\pi}
    \left(
    y_t^{\mathrm{TD}}
    \mid\mathcal{F}_t
    \right)
    =
    \mathrm{Var}_{\pi}
    \left(
    r_{t+1}\mid\mathcal{F}_t
    \right)
    \leq
    \mathrm{Var}_{\pi}
    \left(
    S\mid\mathcal{F}_t
    \right)
    =
    \mathrm{Var}_{\pi}
    \left(
    y_t^{\mathrm{MC}}
    \mid\mathcal{F}_t
    \right).    
\end{equation*}
Thus, the  TD target has lower conditional variance than the MC target. The reduction in conditional target variance is:
\begin{equation*}
    \mathrm{Var}_{\pi}(y_t^{\mathrm{MC}}\mid\mathcal{F}_t)
    -
    \mathrm{Var}_{\pi}(y_t^{\mathrm{TD}}\mid\mathcal{F}_t)
    =
    \mathbb{E}_{\pi}
    \left[
    \mathrm{Var}_{\pi}(S\mid\mathcal{F}_{t+1})
    \mid\mathcal{F}_t
    \right]
    \geq 0.
\end{equation*}


\noindent\textbf{Takeaway.} 
The above analysis  shows that MC regression and ideal TD bootstrapping share the same population success-probability target, while the latter provides lower-variance supervision at intermediate timesteps. 
Intuitively, under MC regression, since each  intermediate state directly inherits the final binary outcome, a state that is close to success can still receive a target of zero if the trajectory eventually fails, and vice versa. Hence, MC regression can be noisy for learning intermediate success probabilities from limited rollout data.  
In contrast, TD bootstrapping uses the next-step success probability as a soft target for the current state, providing less noisy  and more temporally informative supervision than directly assigning the final binary outcome to every intermediate state.
This motivates our use of TD-style bootstrapping for learning eVTA$_0$ from policy rollouts.

\noindent\textbf{Empirical comparison.}
To empirically validate the advantage of TD bootstrapping, we further compare it with MC regression for training eVTA$_0$ on LIBERO. As shown in Table~\ref{tab:td-vs-mc}, TD bootstrapping achieves substantially better temporal consistency, as reflected by higher VOC and VROC scores. Meanwhile, it achieves comparable outcome distinction performance with MC regression. These results indicate that using intermediate bootstrapped success probabilities provides more informative supervision than directly assigning terminal binary outcomes to intermediate steps.

Figure~\ref{fig:td-vs-mc-qualitative} further visualizes reward predictions of the two variants. For the successful trajectory, MC regression assigns consistently high rewards throughout the trajectory, providing limited distinction between early uncertain states and states closer to task completion. In contrast, TD-eVTA$_0$ maintains a lower reward during early stages and increases sharply when the robot successfully grasps and places the object. For the failed trajectory, MC-eVTA$_0$ produces relatively high predictions despite the robot failing to make a successful grasp, whereas TD-eVTA$_0$ maintains low predictions throughout the trajectory. These  results  demonstrate that TD bootstrapping provides more informative intermediate success-probability feedback and better tracks task evolution.


\begin{figure}[t]
    \centering
    \includegraphics[width=\textwidth]{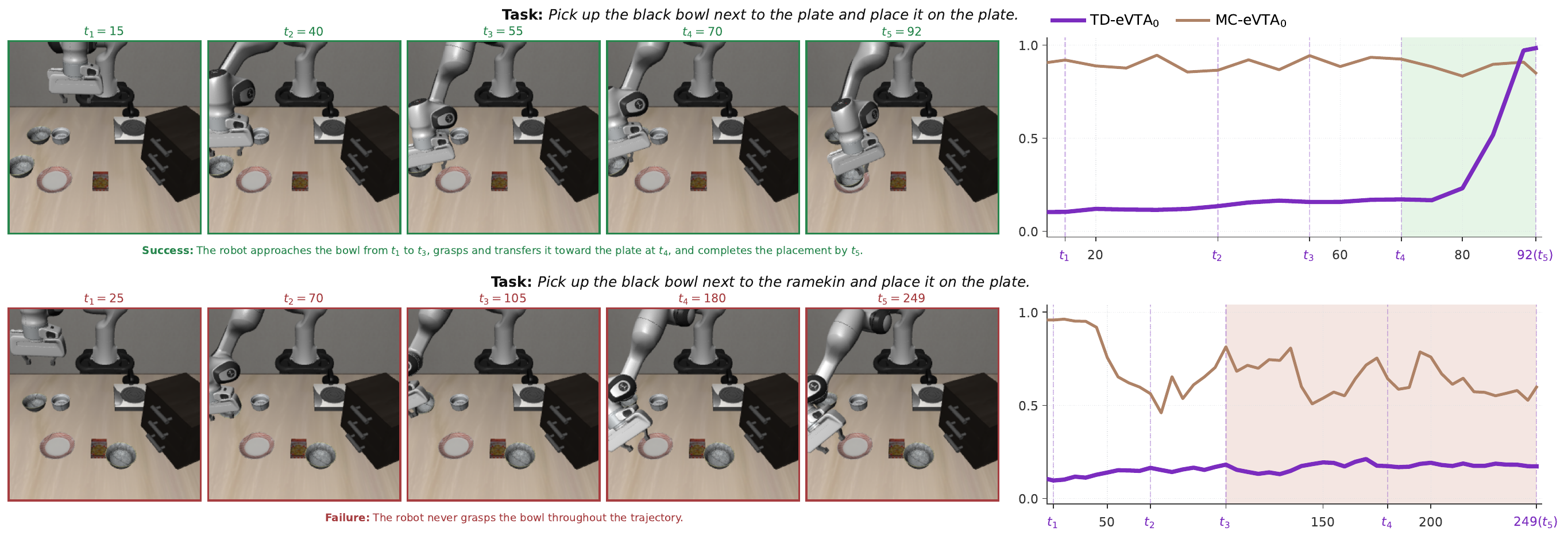}
    \vspace{-0.6cm}
    \caption{Qualitative comparison of eVTA$_0$ trained by TD bootstrapping and MC regression  on a successful trajectory (top) and a failed trajectory (bottom). }
    \label{fig:td-vs-mc-qualitative}
\end{figure}

\stopcontents[appendices]

\end{document}